\documentclass{article}

\usepackage{microtype}
\usepackage{graphicx}
\usepackage{booktabs} % for professional tables

\usepackage{hyperref}
\usepackage{amsfonts}       % blackboard math symbols
\usepackage{url}
\usepackage{amsmath}
\usepackage{amsthm}
\usepackage{wrapfig}
\usepackage{mathtools}
\usepackage{tikz}
\usepackage{subfig}
\usepackage{algorithmic}
\usepackage{subfig}
\usepackage{footmisc}
\newcounter{thm_counter}
\newcounter{lem_counter}

\newcounter{ass_counter}
\numberwithin{ass_counter}{section}

\newtheorem{theorem}[thm_counter]{Theorem}%[section]
\newtheorem{lemma}[lem_counter]{Lemma}%[Lemma]

\newtheorem{assumption}[ass_counter]{Assumption}

\newcommand{\vop}{{\mathcal{T}}}

\newcommand{\ns}{{|\mathcal{S}|}}
\newcommand{\nsa}{{|\mathcal{S}||\mathcal{A}|}}

\newcommand{\tb}[1]{{\textbf{#1}}}

\usepackage{mathtools}
\usepackage{physics}
\usepackage{mathrsfs}
\mathtoolsset{showonlyrefs}
\newcommand{\E}{\mathbb{E}}
\newcommand{\R}{\mathbb{R}}
\newcommand{\Y}{\mathcal{Y}}
\newcommand{\fO}{\mathcal{O}}
\newcommand{\fP}{\mathcal{P}}

\renewcommand{\braket}[2]{\langle {#1}, {#2} \rangle}

\usepackage[inline,ignoremode]{trackchanges}
\addeditor{lb}

\addeditor{sz}

\addeditor{yh}

\usepackage[accepted]{icml2021}

\icmltitlerunning{Breaking the Deadly Triad with a Target Network}

\begin{document}
\twocolumn[
\icmltitle{Breaking the Deadly Triad with a Target Network}

% It is OKAY to include author information, even for blind
% submissions: the style file will automatically remove it for you
% unless you've provided the [accepted] option to the icml2020
% package.

% List of affiliations: The first argument should be a (short)
% identifier you will use later to specify author affiliations
% Academic affiliations should list Department, University, City, Region, Country
% Industry affiliations should list Company, City, Region, Country

% You can specify symbols, otherwise they are numbered in order.
% Ideally, you should not use this facility. Affiliations will be numbered
% in order of appearance, and this is the preferred way.
\icmlsetsymbol{equal}{*}

\begin{icmlauthorlist}
\icmlauthor{Shangtong Zhang}{ox}
\icmlauthor{Hengshuai Yao}{hw,ua}
\icmlauthor{Shimon Whiteson}{ox}
\end{icmlauthorlist}

\icmlaffiliation{ox}{University of Oxford}
\icmlaffiliation{hw}{Huawei Technologies}
\icmlaffiliation{ua}{University of Alberta}

\icmlcorrespondingauthor{Shangtong Zhang}{\\\mbox{shangtong.zhang@cs.ox.ac.uk}}

% You may provide any keywords that you
% find helpful for describing your paper; these are used to populate
% the "keywords" metadata in the PDF but will not be shown in the document
\icmlkeywords{}

\vskip 0.3in
]

% this must go after the closing bracket ] following \twocolumn[ ...

% This command actually creates the footnote in the first column
% listing the affiliations and the copyright notice.
% The command takes one argument, which is text to display at the start of the footnote.
% The \icmlEqualContribution command is standard text for equal contribution.
% Remove it (just {}) if you do not need this facility.

\printAffiliationsAndNotice{}  % leave blank if no need to mention equal contribution
% \printAffiliationsAndNotice{\icmlEqualContribution} % otherwise use the standard text.

\begin{abstract}
The deadly triad refers to the instability of a reinforcement learning algorithm when it employs off-policy learning, 
function approximation,
and bootstrapping simultaneously.
In this paper,
we investigate the target network as a tool for breaking the deadly triad,
providing theoretical support for the conventional wisdom that a target network stabilizes training.
We first propose and analyze a novel target network update rule which augments the commonly used Polyak-averaging style update with two projections. 
We then apply the target network and ridge regularization in several divergent algorithms and show their convergence to regularized TD fixed points.
Those algorithms 
are off-policy with linear function approximation and bootstrapping,
spanning both policy evaluation and control, as well as
both discounted and average-reward settings.
In particular,
we provide the first convergent linear $Q$-learning algorithms under nonrestrictive and changing behavior policies without bi-level optimization.
\end{abstract}

\section{Introduction}
The deadly triad (see, e.g., Chapter 11.3 of \citet{sutton2018reinforcement}) refers to the instability of a value-based reinforcement learning (RL, \citet{sutton2018reinforcement}) algorithm when it employs off-policy learning, function approximation, and 
bootstrapping simultaneously.
Different from \emph{on-policy} methods,
where the policy of interest is executed for data collection,
\emph{off-policy} methods execute a different policy for data collection,
which is usually safer \citep{dulac2019challenges} and more data efficient \citep{lin1992self,sutton2011horde}.
% \notey{``safe" in the sense of ?}
\emph{Function approximation} methods use parameterized functions,
instead of a look-up table,
to represent quantities of interest,
which usually cope better with large-scale problems \citep{mnih2015human,silver2016mastering}.
\emph{Bootstrapping} methods construct update targets for an estimate by using the estimate itself recursively,
which usually has lower variance than \emph{Monte Carlo} methods \citep{sutton1988learning}. 
However,
when an algorithm employs all those three preferred ingredients (off-policy learning, function approximation, and bootstrapping) simultaneously,
there is usually no guarantee that the resulting algorithm is well behaved and the value estimates
can easily diverge (see, e.g., \citet{baird1995residual,tsitsiklis1997analysis,zhang2020average}),
yielding the notorious deadly triad. 

An example of the deadly triad is $Q$-learning \citep{watkins1992q} with linear function approximation,
whose divergence is well documented in \citet{baird1995residual}.
However,
Deep-$Q$-Networks (DQN, \citet{mnih2015human}),
a combination of
$Q$-learning and deep neural network function approximation,
has enjoyed great empirical success.
One major improvement of DQN over linear $Q$-learning is the use of a target network, a copy of the neural network function approximator (the main network) that is periodically synchronized with the main network.
Importantly, 
the bootstrapping target in DQN is computed via the target network instead of the main network.
As the target network changes slowly,
it provides a stable bootstrapping target which in turn stabilizes the training of DQN.
Instead of the periodical synchronization,
\citet{lillicrap2015continuous} propose a Polyak-averaging style target network update,
which has also enjoyed great empirical success \citep{fujimoto2018addressing,haarnoja2018soft}.

Inspired by the empirical success of the target network in RL with deep networks,
in this paper,
we theoretically investigate the target network
as a tool for breaking the deadly triad.
We consider a two-timescale framework,
where the main network is updated faster than the target network.
By using a target network to construct the bootstrapping target,
the main network update becomes least squares regression.
After adding ridge regularization \citep{tikhonov2013numerical} to this least squares problem,
we show convergence for both the target and main networks.

Our main contributions are twofold.
First,
we propose a novel target network update rule augmenting the Polyak-averaging style update with two projections.
The balls for the projections are usually large so most times they are just identity mapping.
However,
those two projections offer significant theoretical advantages making it possible to analyze where the target network converges to (Section~\ref{sec target network}).
Second,
we apply the target network in various existing divergent algorithms and show their convergence to regularized TD \citep{sutton1988learning} fixed points.
% \notey{Does this paper talk about regularization?}
Those algorithms are off-policy algorithms with linear function approximation and bootstrapping,
spanning both policy evaluation and control, as well as
both discounted and average-reward settings.
In particular,
we provide the first convergent linear $Q$-learning algorithms under nonrestrictive and changing behavior policies without bi-level optimization,
for both discounted and average-reward settings.

\section{Background}
Let $M$ be a real positive definite matrix and $x$ be a vector,
% we use $\braket{x}{y}_M \doteq x^\top M y$
 % to denote the inner product induced by $M$ and 
we use $\norm{x}_M \doteq \sqrt{x^\top M x}$ to denote the norm induced by $M$ and $\norm{\cdot}_M$ to denote the corresponding induced matrix norm.
When $M$ is the identity matrix $I$, 
we ignore the subscript $I$ for simplicity.
We use vectors and functions interchangeably when it does not cause confusion, e.g., given $f: \mathcal{X} \to \R$, we also use $f$ to denote the corresponding vector in $\R^{|\mathcal{X}|}$. 
All vectors are column vectors.
We use $\tb{1}$ to denote an all one vector,
whose dimension can be deduced from the context.
$\tb{0}$ is similarly defined.

We consider an infinite horizon Markov Decision Process (MDP, see, e.g., \citet{puterman2014markov}) consisting of a finite state space $\mathcal{S}$,
a finite action space $\mathcal{A}$,
a transition kernel $p: \mathcal{S} \times \mathcal{S} \times \mathcal{A} \to [0, 1]$,
and
a reward function $r: \mathcal{S} \times \mathcal{A} \to \R$.
% a discount factor $\gamma \in [0, 1)$,
% and an initial state distribution $p_0: \mathcal{S} \to [0, 1]$. 
% At time step $t = 0$, an initial state $S_0$ is sampled according to $p_0$.
At time step $t$, 
an agent at a state $S_t$ executes an action $A_t \sim \pi(\cdot | S_t)$,
where $\pi: \mathcal{A} \times \mathcal{S} \to [0, 1]$ is the policy followed by the agent.
The agent then receives a reward $R_{t+1} \doteq r(S_t, A_t)$ and proceeds to a new state $S_{t+1} \sim p(\cdot|S_t, A_t)$.

In the \emph{discounted} setting,
we consider a discount factor $\gamma \in [0, 1)$ and define the return at time step $t$ as
% \begin{align}
$G_t \doteq \sum_{i=1}^\infty \gamma^{i-1} R_{t+i}$,
% \end{align}
which allows us to define the action-value function 
% \begin{align}
$q_\pi(s, a) \doteq \E_{\pi, p}[G_t | S_t = s, A_t = a]$.
% \end{align}
The action-value function $q_\pi$ is the unique fixed point of the Bellman operator $\vop_\pi$, i.e.,
% \begin{align}
$q_\pi = \vop_\pi q_\pi \doteq r + \gamma P_\pi q_\pi$,
% \end{align}
where $P_\pi \in \R^{\nsa \times \nsa}$ is the transition matrix, i.e., $P_\pi((s, a), (s', a')) \doteq \sum_a p(s'|s, a) \pi(a' | s')$.

In the \emph{average-reward} setting,
we assume:
\begin{assumption}
\label{assu ergodic chain target policy}
The chain induced by $\pi$ is ergodic.
\end{assumption}
This allows us to define the \emph{reward rate} 
% \begin{align}
\\$\bar r_\pi \doteq \lim_{T \to \infty} \frac{1}{T}\sum_{t=1}^{T}\E[R_t | p, \pi]$.
% \end{align}
The differential action-value function $\bar q_\pi(s, a)$ is defined as 
\begin{align}
% \bar q_\pi(s, a) \doteq \lim_{T \to \infty} \sum_{t=0}^T \E_{\pi, p}[ r(S_t, A_t) - \bar r_\pi | S_0 = s, A_0 = a].
\textstyle{\lim_{T \to \infty} \sum_{t=0}^T \E_{\pi, p}[ r(S_t, A_t) - \bar r_\pi | S_0 = s, A_0 = a]}.
\end{align}
The differential Bellman equation is
\begin{align}
\label{eq differential bellman}
\bar q = r - \bar r \tb{1} + P_\pi \bar q,
\end{align}
where $\bar q \in \R^\nsa$ and $\bar r \in \R$ are free variables.
It is well known that all solutions to~\eqref{eq differential bellman} form a set $\qty{(\bar q, \bar r) \mid \bar r = \bar r_\pi, \bar q = q_\pi + c \tb{1}, c \in \R}$ \citep{puterman2014markov}.

The \emph{policy evaluation} problem refers to estimating $q_\pi$ or $(\bar q_\pi, \bar r_\pi)$.
The \emph{control} problem refers to finding a policy $\pi$ maximizing $q_\pi(s, a)$ for each $(s, a)$ or maximizing $\bar r_\pi$.
With \emph{linear} function approximation,
we approximate $q_\pi(s, a)$ or $\bar q_\pi(s, a)$ with $x(s, a)^\top w$,
where $x: \mathcal{S} \times \mathcal{A} \to \R^K$ is a feature mapping and $w \in \R^K$ is the learnable parameter.
We use $X \in \R^{\nsa \times K}$ to denote the feature matrix,
each row of which is $x(s, a)^\top$, and assume:
\begin{assumption}
\label{assu full rank X}
 $X$ has linearly independent columns. 
\end{assumption}
In the average-reward setting,
we use an additional parameter $\bar r \in \R$ to approximate $\bar r_\pi$.
In the \emph{off-policy} learning setting, 
the data for policy evaluation or control is collected by executing a policy $\mu$ (behavior policy) in the MDP,
which is different from $\pi$ (target policy).
In the rest of the paper,
we consider the off-policy linear function approximation setting
thus always assume $A_t \sim \mu(\cdot | S_t)$.
We use as shorthand $x_t \doteq x(S_t, A_t), \bar x_t \doteq \sum_a \pi(a | S_t) x(S_t, a)$.

\textbf{Policy Evaluation}. 
In the discounted setting,
similar to Temporal Difference Learning (TD, \citet{sutton1988learning}),
one can use Off-Policy Expected SARSA to estimate $q_\pi$,
which updates $w$ as
\begin{align}
\delta_t &\gets R_{t+1} + \gamma \bar x_{t+1}^\top w_t - x_t^\top w_t, \\
\label{eq TD update}
w_{t+1} &\gets w_t + \alpha_t \delta_t x_t,
\end{align}
where $\qty{\alpha_t}$ are learning rates. 
In the average-reward setting,
\eqref{eq differential bellman} implies that
$\bar r_\pi = d^\top (r + P_\pi \bar q_\pi - \bar q_\pi)$ holds for any probability distribution $d$.
In particular, it holds for $d = d_\mu$.
Consequently,
to estimate $\bar q_\pi$ and $\bar r_\pi$,
\citet{wan2020learning,zhang2020average} update $w$ and $\bar r$ as
\begin{align}
w_{t+1} &\gets w_t + \alpha_t (R_{t+1} - \bar r_t + \bar x_{t+1}^\top w_t - x_t^\top w_t) x_t, \\
\label{eq diff-TD update}
\bar r_{t+1} &\gets \bar r_t + \alpha_t (R_{t+1} + \bar x_{t+1}^\top w_t - x_t^\top w_t - \bar r_t).
\end{align}
Unfortunately,
both~\eqref{eq TD update} and~\eqref{eq diff-TD update} can possibly diverge (see, e.g., \citet{tsitsiklis1997analysis,zhang2020average}),
which exemplifies the deadly triad in discounted and average-reward settings respectively.

\textbf{Control}.
In the discounted setting,
$Q$-learning with linear function approximation yields
\begin{align}
\delta_t &\gets R_{t+1} + \gamma \max_{a'} x(S_{t+1}, a')^\top w_t - x_t^\top w_t, \\
\label{eq Q update}
w_{t+1} &\gets w_t + \alpha_t \delta_t x_t.
\end{align}
In the average-reward setting,
Differential $Q$-learning \citep{wan2020learning} with linear function approximation yields
\begin{align}
\delta_t &\gets R_{t+1} - \bar r_t + \gamma \max_{a'} x(S_{t+1}, a')^\top w_t - x_t^\top w_t, \\
\label{eq diff Q update}
w_{t+1} &\gets w_t + \alpha_t \delta_t x_t, \quad \bar r_{t+1} \gets \bar r_t + \alpha_t \delta_t.
\end{align}
Unfortunately,
both \eqref{eq Q update} and \eqref{eq diff Q update} can possibly diverge as well (see, e.g., \citet{baird1995residual,zhang2020average}),
exemplifying the deadly triad again.

% \citet{mnih2015human} stabilize the training of DQN with a target network.
% \citet{lillicrap2015continuous} propose a variant that gradually updates the target network towards the main network.
Motivated by the empirical success of the target network in deep RL, 
one can apply the target network in the linear function approximation setting.
For example,
using a target network in \eqref{eq Q update} yields
\begin{align}
\delta_t &\gets R_{t+1} + \gamma \max_{a'} x(S_{t+1}, a')^\top \theta_t - x_t^\top w_t, \\
\label{eq target-network q update main}
w_{t+1} &\gets w_t + \alpha_t \delta_t x_t, \\
\label{eq target-network q update target}
\theta_{t+1} &\gets \theta_t + \beta_t (w_t - \theta_t),
\end{align}
where $\theta$ denotes the target network,
$\qty{\beta_t}$ are learning rates,
and we consider the Polyak-averaging style target network update.
The convergence of \eqref{eq target-network q update main} and \eqref{eq target-network q update target}, however,
remains unknown.
Besides target networks, regularization has also been widely used in deep RL, e.g., \citet{mnih2015human} consider a Huber loss instead of a mean-squared loss;
\citet{lillicrap2015continuous} consider $\ell_2$ weight decay in updating  $Q$-values.

\section{Analysis of the Target Network}
\label{sec target network}
In Sections~\ref{sec application oppe} \&~\ref{sec application control},
we consider the merits of using a target network in several linear RL algorithms (e.g., \eqref{eq TD update}~\eqref{eq diff-TD update}~\eqref{eq Q update}~\eqref{eq diff Q update}). 
To this end,
in this section,
we start by proposing and analyzing a novel target network update rule:
\begin{align}
\label{eq general target net update}
\theta_{t+1} \doteq \Gamma_{B_1} \big( \theta_t + \beta_t (\Gamma_{B_2}(w_t) - \theta_{t}) \big).
\end{align}
In~\eqref{eq general target net update},
$w$ denotes the main network and $\theta$ denotes the target network.
$\Gamma_{B_1}: \R^K \to \R^K$ is a projection to the ball $B_1 \doteq \qty{x \in \R^K \mid \norm{x} \leq R_{B_1}}$, i.e.,
\begin{align}
\textstyle{\Gamma_{B_1}(x) \doteq 
x \mathbb{I}_{\norm{x} \leq R_{B_1}} + ({R_{B_1}x}/{\norm{x}}) \mathbb{I}_{\norm{x} > R_{B_1}}},
\end{align}
where $\mathbb{I}$ is the indicator function.
$\Gamma_{B_2}$ is a projection onto the ball $B_2$ with a radius $R_{B_2}$.
We make the following assumption about the learning rates:
\begin{assumption}
\label{assu target net lr}
$\qty{\beta_t}$ is a deterministic positive nonincreasing sequence satisfying $\sum_t \beta_t = \infty, \sum_t \beta_t^2 < \infty$.
\end{assumption}
While \eqref{eq general target net update} specifies only how $\theta$ is updated,
we assume $w$ is updated such that $w$ can track $\theta$ in the sense that
\begin{assumption}
\label{assu target net main net}
There exists $w^*: \R^K \to \R^K$ such that 
$\lim_{t \to \infty} \norm{w_t - w^*(\theta_t)} = 0$ almost surely.
\end{assumption}
After making some additional assumptions on $w^*$,
we arrive at our general convergent results.
\begin{assumption}
\label{assu target net boundedness}
$\sup_\theta \norm{w^*(\theta)} < R_{B_2} < R_{B_1} < \infty$.
\end{assumption}
\begin{assumption}
\label{assu target net contraction}
 $w^*$ is a contraction mapping w.r.t. $\norm{\cdot}$. 
\end{assumption}

\begin{theorem}
\label{thm:target-net}
(Convergence of Target Network)
Under Assumptions~\ref{assu target net lr}-\ref{assu target net contraction},
the iterate $\qty{\theta_t}$ generated by~\eqref{eq general target net update} satisfies
\begin{align}
\textstyle{\lim_{t \to \infty} w_t = \lim_{t \to \infty} \theta_t = \theta^*} \qq{almost surely,} 
\end{align}
where $\theta^*$ is the unique fixed point of $w^*(\cdot)$.
\end{theorem}
Assumptions~\ref{assu target net main net} - \ref{assu target net contraction} are assumed only for now.
Once the concrete update rules for $w$ are specified in the algorithms in Sections~\ref{sec application oppe} \&~\ref{sec application control},
we will prove that those assumptions indeed hold.
Assumption~\ref{assu target net main net} is expected to hold because we will later require that the target network to be updated much slower than the main network.
Consequently, the update of the main network will become a standard least-square regression,
whose solution $w^*$ usually exists.
Assumption~\ref{assu target net contraction} is expected to hold becuase we will later apply ridge regularization to the least-square regression. 
Consequently, its solution $w^*$ will not change too fast w.r.t. the change of the regression target.

The target network update~\eqref{eq general target net update} is the same as that in~\eqref{eq target-network q update target} except for the two projections,
where the first projection $\Gamma_{B_1}$ is standard in optimization literature.
The second projection $\Gamma_{B_2}$,
however,
appears novel and plays a crucial role in our analysis.
% both of which are crucial in our analysis.
% $\Gamma_{B_1}$ is commonly used to make sure the iterate is bounded.
% However, 
\emph{First}, if we have only $\Gamma_{B_1}$,
the iterate $\qty{\theta_t}$ would converge to the invariant set of the ODE
\begin{align}
\label{eq projected ode}
\textstyle{\dv{t} \theta(t) = w^*(\theta(t)) - \theta(t) + \zeta(t)},
\end{align}
where $\zeta(t)$ is a reflection term that moves $\theta(t)$ back to $B_1$ when $\theta(t)$ becomes too large (see, e.g., Section 5 of \citet{kushner2003stochastic}).
Due to this reflection term,
it is possible that $\theta(t)$ visits the boundary of $B_1$ infinitely often.
It thus becomes unclear what the invariant set of~\eqref{eq projected ode} is even if $w^*$ is contractive.
By introducing the second projection $\Gamma_{B_2}$ and ensuring $R_{B_1} > R_{B_2}$,
we are able to remove the reflection term and show that the iterate $\qty{\theta_t}$ tracks the ODE
\begin{align}
\label{eq ode}
\textstyle{\dv{t} \theta(t) = w^*(\theta(t)) - \theta(t)},
\end{align}
whose invariant set is a singleton $\qty{\theta^*}$ when Assumption~\ref{assu target net contraction} holds.
See the proof of Theorem~\ref{thm:target-net} in Section~\ref{sec convergence proof general target net} based on the ODE approach \citep{kushner2003stochastic,borkar2009stochastic}  for more details.
\emph{Second},
to ensure the main network tracks the target network in the sense of Assumption~\ref{assu target net main net} in our applications in Sections~\ref{sec application oppe} \&~\ref{sec application control},
it is crucial that the target network changes sufficiently slowly in the following sense:
\begin{lemma}
\label{lem target net changes slowly}
% $\norm{\theta_{t+1} - \theta_t} \leq \beta_t (R_{B_1} + R_{B_2})$.
$\norm{\theta_{t+1} - \theta_t} \leq \beta_t C_0$ for some constant $C_0 > 0$.
\end{lemma}
Lemma~\ref{lem target net changes slowly} would not be feasible without the second projection $\Gamma_{B_2}$ and we defer its proof to Section~\ref{sec target net changes slowly}

In Sections~\ref{sec application oppe} \&~\ref{sec application control},
we provide several applications of Theorem~\ref{thm:target-net} in both
discounted and average-reward settings,
for both policy evaluation and control.
We consider a two-timescale framework, 
where the target network is updated more slowly than the main network.
Let $\qty{\alpha_t}$ be the learning rates for updating the main network $w$;
we assume
\begin{assumption}
\label{assu target net all lrs}
$\qty{\alpha_t}$ is a deterministic positive nonincreasing sequence satisfying $\sum_t \alpha_t = \infty, \sum_t \alpha_t^2 < \infty$.
Further, for some $d > 0$,
$\sum_t (\beta_t / \alpha_t)^d < \infty$.
\end{assumption}

\section{Application to Off-Policy Policy Evaluation}
\label{sec application oppe}
In this paper, we consider estimating the action-value $q_\pi$ instead of the state-value $v_\pi$ for unifying notations of policy evaluation and control.
The algorithms for estimating $v_\pi$ are straightforward up to change of notations and introduction of importance sampling ratios.

\textbf{Discounted Setting}. Using a target network for bootstrapping in~\eqref{eq TD update} yields 
\begin{align}
w_{t+1} \gets w_t + \alpha_t (R_{t+1} + \gamma \bar x_{t+1}^\top \theta_t - x_t^\top w_t) x_t.
\end{align}
As $\theta_t$ is quasi-static for $w_t$
(Lemma~\ref{lem target net changes slowly} and Assumption~\ref{assu target net all lrs}),
this update becomes least squares regression.
Motivated by the success of ridge regularization in least squares and the widespread use of weight decay in deep RL, 
which is essentially ridge regularization, 
we add ridge regularization to this least squares,
yielding $Q$-evaluation with a Target Network (Algorithm~\ref{alg:off-policy-td}).

\begin{algorithm}
\caption{$Q$-evaluation with a Target Network}
\label{alg:off-policy-td}
\begin{algorithmic}
\STATE
\textbf{INPUT:} $\eta > 0, R_{B_1} > R_{B_2} > 0$\\
\STATE Initialize $\theta_0 \in B_1$ and $S_0$
\STATE Sample $A_0 \sim \mu(\cdot | S_0)$
\FOR{$t = 0, 1, \dots$}
\STATE Execute $A_t$, get $R_{t+1}$ and $S_{t+1}$
\STATE Sample $A_{t+1} \sim \mu(\cdot | S_{t+1})$ 
\STATE $\bar x_{t+1} \doteq \sum_{a'} \pi(a'| S_{t+1}) x(S_{t+1}, a')$
\STATE $\delta_t \doteq R_{t+1} + \gamma \bar x_{t+1}^\top \theta_t - x_t^\top w_t$
\STATE $w_{t+1} \doteq w_t + \alpha_t \delta_t x_t - \alpha_t \eta w_t$ 
\STATE $\theta_{t+1} \doteq \Gamma_{B_1} \big( \theta_t + \beta_t (\Gamma_{B_2}(w_t) - \theta_{t}) \big)$
\ENDFOR
\end{algorithmic}
\end{algorithm}

Let $A = X^\top D_\mu (I - \gamma P_\pi)X, b = X^\top D_\mu r$,
where $D_\mu$ is a diagonal matrix whose diagonal entry is $d_\mu$,
the stationary state-action distribution of the chain induced by $\mu$.
Let $\Pi_{D_\mu} \doteq X (X^\top D_\mu X)^{-1}X^\top D_\mu$ be the projection to the column space of $X$.
We have

\begin{assumption}
\label{assu:op-td-ergodic}
The chain in $\mathcal{S} \times \mathcal{A}$ induced by $\mu$ is ergodic.
\end{assumption}
% \begin{assumption}
% \label{assu:op-td-nonsingular}
% $A$ is nonsingular.
% \end{assumption}
\begin{theorem}
\label{thm:op-td}
Under Assumptions~\ref{assu full rank X},~\ref{assu target net lr},~\ref{assu target net all lrs}, \&~\ref{assu:op-td-ergodic},
for any $\xi \in (0, 1)$,
let 
% \begin{align}
$\textstyle{C_0 \doteq \frac{2(1 - \xi)\sqrt{\eta}}{\gamma \norm{P_\pi}_{D_\mu}}, C_1 \doteq \frac{\norm{r}}{2\xi\sqrt{\eta}} + 1}$,
% \end{align}
then for all $\norm{X} < C_0, C_1 < R_{B_1}, R_{B_1} - \xi < R_{B_2} < R_{B_1}$
the iterate $\qty{w_t}$ generated by Algorithm~\ref{alg:off-policy-td} satisfies 
\begin{align}
\textstyle{\lim_{t \to \infty} w_t = w^*_\eta} \qq{almost surely,}
\end{align}
where $w^*_\eta$ is the unique solution of 
% \begin{align}
$(A+\eta I) w - b = \tb{0}$,
% \end{align}
and
\begin{align}
&\norm{X w^*_\eta - q_\pi} \\
\leq & \textstyle{ ( \frac{\sigma_{\max}(X)^2}{\sigma_{\min}(X)^4 \sigma_{\min}(D_\mu)^{2.5}} \norm{q_\pi} \eta + \norm{\Pi_{D_\mu} q_\pi - q_\pi} ) / \xi},
\end{align}
where $\sigma_{\max}(\cdot), \sigma_{\min}(\cdot)$ denotes the largest and minimum singular values.
\end{theorem}
We defer the proof to Section~\ref{sec:proof-of-td}.
Theorem~\ref{thm:op-td} requires that the balls for projection are sufficiently large,
which is completely feasible in practice.
Theorem~\ref{thm:op-td} also requires that the feature norm $\norm{X}$ is not too large.
Similar assumptions on feature norms also appear in \citet{,zou2019finite,du2019provably,chen2019,carvalho2020new,wang2020finite,wu2020finite} and can be easily achieved by scaling. 
% More importantly, the requirement on $\norm{X}$ in Theorem~\ref{thm:op-td}, as well as all the theorems in the rest of this paper, 
% is only used to fulfill Assumption~\ref{assu target net contraction},
% without which $\qty{\theta_t}$ in~\eqref{eq general target net update} still converges,
% but to the invariant set of the ODE~\eqref{eq ode}.
% This requirement is to investigate where the iterate converges to instead of whether it converges or not. 

The solutions to $Aw - b = \tb{0}$, if they exist,
are TD fixed points for off-policy policy evaluation in the discounted setting \citep{sutton2009convergent,sutton2009fast}.
Theorem~\ref{thm:op-td} shows that Algorithm~\ref{alg:off-policy-td} finds a regularized TD fixed point $w^*_\eta$,
which is also the solution of  
Least-Squares TD methods (LSTD, \citet{boyan1999least,yu2010convergence}). 
% As noted in Section 9.8 of \citet{sutton2018reinforcement},
% this $\eta$ in LSTD plays a key role in its performance.
LSTD maintains estimates for $A$ and $b$ (referred to as $\hat A$ and $\hat b$) in an online fashion,
which requires $\mathcal{O}(K^2)$ computational and memory complexity per step.
As $\hat A$ is not guaranteed to be invertible,
LSTD usually uses $(\hat A + \eta I)^{-1} \hat b$ as the solution and $\eta$ plays a key role in its performance (see, e.g, Chapter 9.8 of \citet{sutton2018reinforcement}).
By contrast, Algorithm~\ref{alg:off-policy-td} finds the LSTD solution (i.e., $w^*_\eta$) with only $\mathcal{O}(K)$ computational and memory complexity per step.
Moreover,
Theorem~\ref{thm:op-td} provides a performance bound for $w^*_\eta$.
Let $w^*_0 \doteq A^{-1}b$;
\citet{kolter2011fixed} shows with a counterexample that the approximation error of TD fixed points
(i.e., $\norm{Xw^*_0 - q_\pi}$) can be arbitrarily large if $\mu$ is far from $\pi$,
as long as there is representation error (i.e., $\norm{\Pi_{D_\mu} q_\pi - q_\pi} > 0$) (see Section~\ref{sec expts} for details).
By contrast,
Theorem~\ref{thm:op-td} guarantees that
$\norm{Xw^*_\eta - q_\pi}$ is bounded from above,
which is one possible advantage of regularized TD fixed points. 

% Algorithm~\ref{alg:off-policy-td} is a warm-up showing how target network,
% together with ridge regularization,
% can be used to break the deadly triad.
% In the rest of the paper,
% we will apply similar techniques in more challenging settings (average reward and control).

\begin{algorithm}[h]
\caption{Diff. $Q$-evaluation with a Target Network}
\label{alg:differential-off-policy-td}
\begin{algorithmic}
\STATE
\textbf{INPUT:} $\eta > 0, R_{B_1} > R_{B_2} > 0$\\
\STATE Initialize $[\theta^r_0,  \theta^{w\top}_0]^\top \in B_1$ and $S_0$
\STATE Sample $A_0 \sim \mu(\cdot | S_0)$
\FOR{$t = 0, 1, \dots$}
\STATE Execute $A_t$, get $R_{t+1}$ and $S_{t+1}$
\STATE Sample $A_{t+1} \sim \mu(\cdot | S_{t+1})$ 
\STATE $\bar x_{t+1} \doteq \sum_{a'} \pi(a'| S_{t+1}) x(S_{t+1}, a')$
\STATE $\delta_t \doteq R_{t+1} - \theta^r_t + \bar{x}_{t+1}^\top \theta^w_t - x_t^\top w_t$
\STATE $w_{t+1} \doteq w_t + \alpha_t \delta_t x_t - \alpha_t \eta w_t$
\STATE $\bar r_{t+1} \doteq \bar r_t + \alpha_t (R_{t+1} + \bar x_{t+1}^\top \theta_t^w - x_t^\top \theta_t^w - \bar r_t)$ 
\STATE $\mqty[\theta^r_{t+1} \\ \theta^w_{t+1}] \doteq \Gamma_{B_1} \big(\mqty[\theta^r_{t} \\ \theta^w_{t}] + \beta_t (\Gamma_{B_2}(\mqty[\bar r_t \\ w_t]) - \mqty[\theta^r_{t} \\ \theta^w_{t}]) \big)$
\ENDFOR
\end{algorithmic}
\end{algorithm}
\textbf{Average-reward Setting}.
In the average-reward setting,
we need to learn both $\bar r$ and $w$.
Hence,
we consider target networks $\theta^r$ and $\theta^w$ for $\bar r$ and $w$ respectively.
Plugging $\theta^r$ and $\theta^w$ into~\eqref{eq diff-TD update} for bootstrapping yields
Differential $Q$-evaluation with a Target Network (Algorithm~\ref{alg:differential-off-policy-td}),
where $\qty{B_i}$ are now balls in $R^{K+1}$.
In Algorithm~\ref{alg:differential-off-policy-td}, 
we impose ridge regularization only on $w$ as $\bar r$ is a scalar and thus does not have any representation capacity limit.

% The feature matrix now becomes $Z \doteq \mqty[\zeta \tb{1},  X]$. 
% After making a standard assumption (see, e.g., \citet{tsitsiklis1999average,abbasi2019politex,zhang2020average}) to ensure $Z$ has linearly independent columns,
% we arrive at our convergent result.
% \begin{assumption}
% \label{assu Z full rank}
% For any $c \in \R, w \in \R^K, Xw \neq c \tb{1}$.
% \end{assumption}
\begin{theorem}
\label{thm:diff-op-td}
Under Assumptions~\ref{assu ergodic chain target policy},~\ref{assu full rank X},~\ref{assu target net lr},~\ref{assu target net all lrs}, \&~\ref{assu:op-td-ergodic},
for any $\xi \in (0, 1)$,
there exist constants $C_0$ and $C_1$ such that
for all $\norm{X} < C_0, C_1 < R_{B_1}, R_{B_1} - \xi < R_{B_2} < R_{B_1}$,
the iterates $\qty{\bar r_t}$ and $\qty{w_t}$ generated by Algorithm~\ref{alg:differential-off-policy-td} satisfy 
\begin{align}
\textstyle{\lim_{t \to \infty} \bar r_t} &= \textstyle{d_\mu^\top (r + P_\pi Xw^*_\eta - Xw^*_\eta)}, \\
\textstyle{\lim_{t \to \infty} w_t} &= w^*_\eta \qq{almost surely,}
\end{align}
where $w^*_\eta$ is the unique solution of 
% \begin{align}
$(\bar A+\eta I) w - \bar b = \tb{0}$
% \end{align}
with
\begin{align}
\bar A &\doteq X(D_\mu - d_\mu d_\mu^\top) (I - P_\pi) X, \\
\bar b &\doteq X^\top (D_\mu - d_\mu d_\mu^\top)r.
\end{align}
If features are zero-centered (i.e., $X^\top d_\mu = \tb{0}$), then
\begin{align}
\norm{X w^*_\eta - \bar q_\pi^c} &\leq ( \textstyle{ \frac{\sigma_{\max}(X)^2}{\sigma_{\min}(X)^4 \sigma_{\min}(D_\mu)^{2.5}} } \norm{\bar q_\pi^c} \eta  \\
&\quad + \norm{\Pi_{D_\mu} \bar q_\pi^c - \bar q_\pi^c}) / \xi, \\
|\bar r^*_\eta - \bar r_\pi| &\leq \norm{d_\mu^\top (P_\pi - I)} \inf_c \norm{(Xw^*_\eta - \bar q_\pi^c)}, 
\end{align}
where $\bar q_\pi^c \doteq \bar q_\pi + c \tb{1}$.
\end{theorem}
We defer the proof to Section~\ref{sec:proof-diff-of-td}.
As the differential Bellman equation~\eqref{eq differential bellman} has infinitely many solutions for $\bar q$, 
all of which differ only by some constant offsets,
we focus on analyzing the quality of $Xw^*_\eta$ w.r.t. $\bar q_\pi^c$ in Theorem~\ref{thm:diff-op-td}.
The zero-centered feature assumption is also used in \citet{zhang2020average},
which can be easily fulfilled in practice by subtracting all features with the estimated mean.
In the on-policy case (i.e., $\mu = \pi$), we have $d_\mu^\top (P_\pi - I) = \tb{0}$,
indicating $\bar r^*_\eta = \bar r_\pi$,
i.e., the regularization on the value estimate does not pose any bias on the reward rate estimate.

As shown by \citet{zhang2020average},
if the update~\eqref{eq diff-TD update} converges,
it converges to
$w^*_0$, 
the TD fixed point for off-policy policy evaluation in the average-reward setting, 
which satisfies $\bar A w^*_0 + \bar b = 0$.
Theorem~\ref{thm:diff-op-td} shows that Algorithm~\ref{alg:differential-off-policy-td} converges to a regularized TD fixed point. 
Though \citet{zhang2020average} give a bound on $\norm{Xw^*_0 - \bar q_\pi^c}$,
their bound holds only if $\mu$ is sufficiently close to $\pi$.
By contrast,
our bound on $w^*_\eta$ in Theorem~\ref{thm:diff-op-td} holds for all $\mu$.

\section{Application to Off-Policy Control}
\label{sec application control}
% We now extend the policy evaluation algorithms in Section~\ref{sec application oppe} to the control setting.
% We use $\pi_w$ and $\mu_w$ to denote the target policy and behavior policy respectively.
% The subscript $w$ indicates their dependency on the action-value estimate $x(s, a)^\top w$.
% e.g., $\pi_w$ can be a greedy policy with random tie breaking or a softmax policy w.r.t. $x(s, a)^\top w$,
% $\mu_w$ can be an $\epsilon$-greedy policy or any other exploratory policy.

\textbf{Discounted Setting}.
Introducing a target network and ridge regularization in~\eqref{eq Q update} yields $Q$-learning with a Target Network (Algorithm~\ref{alg:linear-q}),
where the behavior policy $\mu_\theta$ depends on $\theta$ through the action-value estimate $X\theta$ and can be any policy satisfying the following two assumptions. 
\begin{assumption}
\label{assu:control-ergodic}
Let $\mathcal{P}$ be the closure of $\qty{P_{\mu_\theta} \mid \theta \in \R^K}$.
For any $P \in \mathcal{P}$, the Markov chain evolving in $\mathcal{S} \times \mathcal{A}$ induced by $P$ is ergodic.
\end{assumption}
\begin{assumption}
\label{assu:control-parameterization}
$\mu_\theta(a|s)$ is Lipschitz continuous in $X_s \theta$,
where $X_s \in \R^{|\mathcal{A}| \times K}$ is the feature matrix for the state $s$, i.e., its $a$-th row is $x(s, a)^\top$.
\end{assumption}
Assumption~\ref{assu:control-ergodic} is standard.
When the behavior policy $\mu$ is fixed (independent of $\theta$),
the induced chain is usually assumed to be ergodic when analyzing the behavior of $Q$-learning (see, e.g., \citet{melo2008analysis,chen2019,cai2019neural}).
In Algorithm~\ref{alg:linear-q},
the behavior policy $\mu_\theta$ changes every step,
so it is natural to assume that any of those behavior policies induces an ergodic chain.
A similar assumption is also used by \citet{zou2019finite} in their analysis of on-policy linear SARSA.
Moreover, \citet{zou2019finite} assume not only the ergodicity but also the \emph{uniform} ergodicity of their sampling policies.
Similarly,
in Assumption~\ref{assu:control-ergodic},
we assume ergodicity for not only all the transition matrices,
but also their limits (c.f. the closure $\fP$).
A similar assumption is also used by \citet{marbach2001simulation} in their analysis of on-policy actor-critic methods.
Assumption~\ref{assu:control-parameterization} can be easily fulfilled,
e.g.,
by using a softmax policy w.r.t. $x(s, \cdot)^\top \theta$. 
\begin{algorithm}
\caption{$Q$-learning with a Target Network}
\label{alg:linear-q}
\begin{algorithmic}
\STATE
\textbf{INPUT:} $\eta > 0, R_{B_1} > R_{B_2} > 0$\\
\STATE Initialize $\theta_0 \in B_1$ and $S_0$
\STATE Sample $A_0 \sim \mu_{\theta_0}(\cdot | S_0)$
\FOR{$t = 0, 1, \dots$}
\STATE Execute $A_t$, get $R_{t+1}$ and $S_{t+1}$
\STATE Sample $A_{t+1} \sim \mu_{\theta_t}(\cdot | S_t)$ 
\STATE $\delta_t \doteq R_{t+1} + \gamma \max_{a'} x(S_{t+1}, a')^\top \theta_t - x_t^\top w_t$
\STATE $w_{t+1} \doteq w_t + \alpha_t \delta_t x_t - \alpha_t \eta w_t$ 
\STATE $\theta_{t+1} \doteq \Gamma_{B_1} \big( \theta_t + \beta_t (\Gamma_{B_2}(w_t) - \theta_{t}) \big)$
\ENDFOR
\end{algorithmic}
\end{algorithm}

\begin{theorem}
\label{thm:linear-q}
Under Assumptions~\ref{assu full rank X},~\ref{assu target net lr},~\ref{assu target net all lrs},~\ref{assu:control-ergodic}, \&~\ref{assu:control-parameterization},
for any $\xi \in (0, 1), R_{B_1} > R_{B_2} > R_{B_1} - \xi > 0$,
there exists a constant $C_0$ such that for all $\norm{X} < C_0$,
the iterate $\qty{w_t}$ generated by Algorithm~\ref{alg:linear-q} satisfies 
\begin{align}
\textstyle{ \lim_{t \to \infty} w_t = w^*_\eta} \qq{almost surely,}
\end{align}
where $w^*_\eta$ is the unique solution of 
\begin{align}
\label{eq linear q obj}
(A_{\pi_w, \mu_w}+\eta I) w - b_{\mu_w} = \tb{0}
\end{align}
inside $B_1$.
Here 
\begin{align}
A_{\pi_w, \mu_w} &\doteq X^\top D_{\mu_w} (I - \gamma P_{\pi_w}) X, \, b_{\mu_w} \doteq  X^\top D_{\mu_w} r,
\end{align}
and $\pi_w$ denotes the greedy policy w.r.t. $x(s, \cdot)^\top w$.
\end{theorem}
We defer the proof to Section~\ref{sec:proof-of-linear-q}.
Analogously to the policy evaluation setting,
if we call the solutions of $A_{\pi_w, \mu_w} w - b_{\mu_w} = \tb{0}$ TD fixed points for control in the discounted setting,
then Theorem~\ref{thm:linear-q} asserts that Algorithm~\ref{alg:linear-q} finds a regularized TD fixed point.

Algorithm~\ref{alg:linear-q} and Theorem~\ref{thm:linear-q} are significant in two aspects.
\emph{First},
in Algorithm~\ref{alg:linear-q},
the behavior policy is a function of the target network and thus changes every time step.
By contrast,
previous work on $Q$-learning with function approximation (e.g., \citet{melo2008analysis,maei2010toward,chen2019,cai2019neural,chen2019zap,lee2019unified,xu2020finite,carvalho2020new,wang2020finite}) usually assumes the behavior policy is fixed.
Though \citet{fan2020theoretical} also adopt a changing behavior policy,
they consider bi-level optimization.
At each time step, the nested optimization problem must be solved exactly,
which is computationally expensive and sometimes unfeasible.
To the best of our knowledge,
we are the first to analyze $Q$-learning with function approximation under a changing behavior policy and without nested optimization problems.
Compared with the fixed behavior policy setting or the bi-level optimization setting,
our two-timescale setting with a changing behavior policy is more closely related to actual practice (e.g., \citet{mnih2015human,lillicrap2015continuous}).

\emph{Second},
Theorem~\ref{thm:linear-q} does not enforce any similarity between $\mu_\theta$ and $\pi_w$;
they can be arbitrarily different.
By contrast,
previous work (e.g., \citet{melo2008analysis,chen2019,cai2019neural,xu2020finite,lee2019unified}) usually requires the strong assumption that the fixed behavior policy $\mu$ is sufficiently close to the target policy $\pi_w$. 
As the target policy (i.e., the greedy policy) can change every time step due to the changing action-value estimates,
this strong assumption rarely holds.
While some work removes this strong assumption, it introduces other problems instead.
In Greedy-GQ, \citet{maei2010toward} avoid this strong assumption by computing sub-gradients of an MSPBE objective
% \begin{align}
$\text{MSPBE}(w) \doteq \norm{A_{\pi_w, \mu} w - b_\mu}_{C_{\mu}^{-1}}^2$
% \end{align}
directly,
where $C_\mu \doteq X^\top D_\mu X$.
If linear $Q$-learning~\eqref{eq Q update} under a fixed behavior policy $\mu$ converges,
it converges to the minimizer of MSPBE$(w)$.
Greedy-GQ, however,
converges only to a stationary point of MSPBE$(w)$.
By contrast,
Algorithm~\ref{alg:linear-q} converges to a minimizer of our regularized MSPBE (c.f.~\eqref{eq linear q obj}).
% (it can potentially have multiple minimizers).
In Coupled $Q$-learning,
\citet{carvalho2020new} avoid this strong assumption by using a target network as well, which they update as
\begin{align}
\label{eq:cq-target}
\theta_{t+1} \gets \theta_t + \alpha_t ((x_t x_t^\top) w_t - \theta_t).
\end{align}
% \notey{This update is very similar to my LMS-2\url{https://hengshuaiyao.github.io/papers/lms2.pdf}, except there we used $u_t$ on the right hand.}
This target network update deviates much from the commonly used Polyak-averaging style update,
while our~\eqref{eq general target net update} is identical to the Polyak-averaging style update most times if the balls for projection are sufficiently large.  
Coupled $Q$-learning updates the main network $w$ as usual (see~\eqref{eq target-network q update main}).
With the Coupled $Q$-learning updates~\eqref{eq target-network q update main} and~\eqref{eq:cq-target},
\citet{carvalho2020new} prove that the main network and the target network converge to $\bar{w}$ and $\bar{\theta}$ respectively,
which satisfy
\begin{align}
X\bar{w} = X X^\top D_\mu \mathcal{T}_{\pi_{\bar{w}}} X\bar{w}, \quad X\bar{\theta} &= \Pi_{D_\mu} \mathcal{T}_{\pi_{\bar{w}}} X\bar{w}.
\end{align}
% \notey{Make sure the second is correct? I need to check this. }
It is, however, not clear how $\bar{w}$ and $\bar{\theta}$ relate to TD fixed points. 
% \notey{Can we understand this better?}
\citet{yang2019convergent} also use a target network to avoid this strong assumption.
Their target network update is the same as~\eqref{eq general target net update} except that they have only one projection $\Gamma_{B_1}$.
Consequently,
they face the problem of the reflection term $\zeta(t)$ (c.f.~\eqref{eq projected ode}).
They also assume the main network $\qty{w_t}$ is always bounded,
 a strong assumption that we do not require.
Moreover,
they consider a fixed sampling distribution for obtaining i.i.d. samples,
while our data collection is done by executing the changing behavior policy $\mu_\theta$ in the MDP. 

One limit of Theorem~\ref{thm:linear-q} is that the bound on $\norm{X}$ (i.e., $C_0$) depends on $1 / R_{B_1}$ (see the proof in Section~\ref{sec:proof-of-linear-q} for the analytical expression),
which means $C_0$ could potentially be small.
Though we can use a small $\eta$ accordingly to ensure that the regularization effect of $\eta$ is modest,
a small $C_0$ may not be desirable in some cases.
To address this issue,
we propose Gradient $Q$-learning with a Target Network,
inspired by Greedy-GQ.
We first equip MSPBE$(w)$ with a changing behavior policy $\mu_w$, yielding the following objective
% \begin{align}
% \label{eq td fixed point control}
$\norm{A_{\pi_w, \mu_w} w - b_{\mu_w}}_{C_{\mu_w}^{-1}}^2$.
% \end{align}
We then use the target network $\theta$ in place of $w$ in the non-convex components, yielding
\begin{align}
\label{eq obj diff graident q}
L(w, \theta) \doteq \textstyle{\norm{A_{\pi_\theta, \mu_\theta} w - b_{\mu_\theta}}_{C_{\mu_\theta}^{-1}}^2 + \eta \norm{w}^2},
\end{align}
where we have also introduced a ridge term.
At time step $t$,
we update $w_t$ following the gradient $\nabla_w L(w, \theta_t)$ and update the target network $\theta_t$ as usual.
Details are provided in Algorithm~\ref{alg:gradient-q},
where the additional weight vector $u \in \R^K$ results from a weight duplication trick (see \citet{sutton2009convergent,sutton2009fast} for details) to address a double sampling issue in estimating $\nabla_w L(w, \theta)$.

\begin{algorithm}
\caption{Gradient $Q$-learning with a Target Network}
\label{alg:gradient-q}
\begin{algorithmic}
\STATE
\textbf{INPUT:} $\eta > 0, R_{B_1} > R_{B_2} > 0$\\
\STATE Initialize $\theta_0 \in B_1$ and $S_0$
\STATE Sample $A_0 \sim \mu_{\theta_0}(\cdot | S_0)$
\FOR{$t = 0, 1, \dots$}
\STATE Execute $A_t$, get $R_{t+1}$ and $S_{t+1}$
\STATE Sample $A_{t+1} \sim \mu_{\theta_t}(\cdot | S_t)$ 
\STATE $\bar x_{t+1} \doteq \sum_{a'} \pi_{\theta_t}(a' | S_{t+1}) x(S_{t+1}, a')$
% \notey{$s'$ should be $S_{t+1}$}
\STATE $\delta_t \doteq R_{t+1} + \gamma \bar x_{t+1}^\top w_t - x_t^\top w_t$
% \notey{Have a comment here that $x_{t+1}^\top w_t$ is not $x_{t+1}^\top u_t$}
\STATE $u_{t+1} \doteq u_t + \alpha_t (\delta_t - x_t^\top u_t) x_t$
\STATE $w_{t+1} \doteq w_t + \alpha_t (x_t - \gamma \bar x_{t+1}) x_t^\top u_t - \alpha_t \eta w_t$ 
% \notey{Here $\bar x_{t+1}$ forgot $\gamma$. }
\STATE $\theta_{t+1} \doteq \Gamma_{B_1} \big( \theta_t + \beta_t (\Gamma_{B_2}(w_t) - \theta_{t}) \big)$
\ENDFOR
\end{algorithmic}
\end{algorithm}

In Algorithm~\ref{alg:linear-q},
the target policy $\pi_w$ is a greedy policy,
which is not continuous in $w$.
This discontinuity is not a problem there but requires sub-gradients in the analysis of Algorithm~\ref{alg:gradient-q},
which complicates the presentation.
We, therefore, impose Assumption~\ref{assu:control-parameterization} on $\pi_w$ as well.
% make the following assumption on $\pi_\theta$.
% A softmax policy w.r.t. $x(s, \cdot)^\top \theta$ is an example satisfying this assumption.
\begin{assumption}
\label{assu target policy continuous}
$\pi_\theta(a|s)$ is Lipschitz continuous in $X_s\theta$.
\end{assumption}
Though a greedy policy no longer satisfies Assumption~\ref{assu target policy continuous},
we can simply use a softmax policy with any temperature.

\begin{theorem}
\label{thm:gradient-q}
Under Assumptions~\ref{assu full rank X},~\ref{assu target net lr},~\ref{assu target net all lrs}, \&~\ref{assu:control-ergodic}-\ref{assu target policy continuous},
there exist positive constants $C_0$ and $C_1$ such that for all
% \begin{align}
$\norm{X} < C_0, R_{B_1} > R_{B_2} > C_1$,
% \end{align}
the iterate $\qty{w_t}$ generated by Algorithm~\ref{alg:gradient-q} satisfies 
\begin{align}
\textstyle{\lim_{t \to \infty} w_t = w^*_\eta} \qq{almost surely,}
\end{align}
where $w^*_\eta$ is the unique solution of 
\begin{align}
(A_{\pi_w, \mu_w}^\top C_{\mu_w}^{-1} A_{\pi_w, \mu_w} + \eta I)w = A_{\pi_w, \mu_w}^\top C_{\mu_w}^{-1} b_{\mu_w}.
\end{align}
% \begin{align}
% \label{eq gradient q fixed point}
% &(A_{\pi_{w^*_\eta}, \mu_{w^*_\eta}}^\top C_{\mu_{w^*_\eta}}^{-1} A_{\pi_{w^*_\eta}, \mu_{w^*_\eta}} + \eta I)^{-1} w^*_\eta \\
% =& A_{\pi_{w^*_\eta}, \mu_{w^*_\eta}}^\top C_{\mu_{w^*_\eta}}^{-1} b_{\mu_{w^*_\eta}}.
% \end{align}
\end{theorem}
We defer the proof to Section~\ref{sec:proof-of-gradient-q}.
Importantly, the $C_0$ here does not depend on $R_{B_1}$ and $R_{B_2}$.
More importantly,
the condition on $\norm{X}$ (or equivalently, $\eta$) in Theorem~\ref{thm:gradient-q} is only used to fulfill Assumption~\ref{assu target net contraction},
without which $\qty{\theta_t}$ in Algorithm~\ref{alg:gradient-q} still converges to an invariant set of the ODE~\eqref{eq ode}.
This condition is to investigate where the iterate converges to instead of whether it converges or not.
If we assume $w^*_0 \doteq \lim_{\eta \to 0} w^*_\eta$ exists and $A_{\pi_{w^*_0}, \mu_{w^*_0}}$ is invertible,
we can see
% \begin{align}
$A_{\pi_{w^*_0}, \mu_{w^*_0}} w^*_0 - b_{\mu_{w^*_0}} = \tb{0}$,
% \end{align}
indicating $w^*_0$ is a TD fixed point.
$w^*_\eta$ can therefore be regarded as a regularized TD fixed point,
though how the regularization is imposed here (c.f.~\eqref{eq obj diff graident q}) is different from that in Algorithm~\ref{alg:linear-q} (c.f.~\eqref{eq linear q obj}).

\textbf{Average-reward Setting}.
% \notey{The algorithm should be described in fine details. 
% The goal of average-reward learning, etc. 
% The quantities of $\theta^r$ and $\theta^w$: their roles. 
% The $\bar{r}$ signal tracks the TD error. 
% It is an interesting quantity deserving a nice explanation. 
% }
Similar to Algorithm~\ref{alg:differential-off-policy-td},
introducing a target network and ridge regularization in~\eqref{eq diff Q update}
yields Differential $Q$-learning with a Target Network (Algorithm~\ref{alg:diff-linear-q}).
Similar to Algorithm~\ref{alg:differential-off-policy-td}, $\qty{B_i}$ are now balls in $R^{K+1}$.

\begin{algorithm}
\caption{Diff. $Q$-learning with a Target Network}
\label{alg:diff-linear-q}
\begin{algorithmic}
\STATE
\textbf{INPUT:} $\eta > 0, R_{B_1} > R_{B_2} > 0$\\
\STATE Initialize $[\theta^r_0,  \theta^{w\top}_0]^\top \in B_1$ and $S_0$
\STATE Sample $A_0 \sim \mu(\cdot | S_0)$
\FOR{$t = 0, 1, \dots$}
\STATE Execute $A_t$, get $R_{t+1}$ and $S_{t+1}$
\STATE Sample $A_{t+1} \sim \mu_{\theta_t^w}(\cdot | S_{t+1})$ 
\STATE $\delta_t \doteq R_{t+1} - \theta^r_t + \max_{a'} x(S_{t+1}, a')^\top \theta^w_t - x_t^\top w_t$
\STATE $w_{t+1} \doteq w_t + \alpha_t \delta_t x_t - \alpha_t \eta w_t$
% \notey{$w_t$ seems an average of TD update with regularization. }
% \notey{$\theta^w$: notation is bit odd.}
\STATE $\delta_t' \doteq R_{t+1} + \max_{a'} x(S_{t+1}, a')^\top \theta^w_t - x_t^\top \theta^w_t - \bar r_t$
% \notey{Two TD errors here. They are worth an explanation. 
% I think it's better to write them in similar format. 
%  \[
%  \delta_t \doteq R_{t+1} - \theta^r_t + \max_{a'} x(S_{t+1}, a')^\top \theta^w_t - x_t^\top w_t
%  \]
%  \[
%  \delta_t' \doteq R_{t+1}- \bar r_t  + \max_{a'} x(S_{t+1}, a')^\top \theta^w_t - x_t^\top \theta^w_t 
%  \]
% Then explain their roles. 
% My understanding is that $ \delta_t'$ is computed using more stable signals which might provide slower but more accurate tracking. 
% }
\STATE $\bar r_{t+1} \doteq \bar r_t + \alpha_t \delta_t'$ 
\STATE $\mqty[\theta^r_{t+1} \\ \theta^w_{t+1}] \doteq \Gamma_{B_1} \big(\mqty[\theta^r_{t} \\ \theta^w_{t}] + \beta_t (\Gamma_{B_2}(\mqty[\bar r_t \\ w_t]) - \mqty[\theta^r_{t} \\ \theta^w_{t}]) \big)$
% \notey{$\theta^w_{t}$ averages over $w_t$: it is the target network parameter. You need to say it. (Maybe in the average-reward-setting section. )
% Then maybe say something about the reason that you also need a target network parameter, $\theta^r$. 
% I guess because you need to estimate the average reward. 
% }
\ENDFOR
\end{algorithmic}
\end{algorithm}
\begin{figure*}[t]
\subfloat[]{\includegraphics[width=0.25\linewidth]{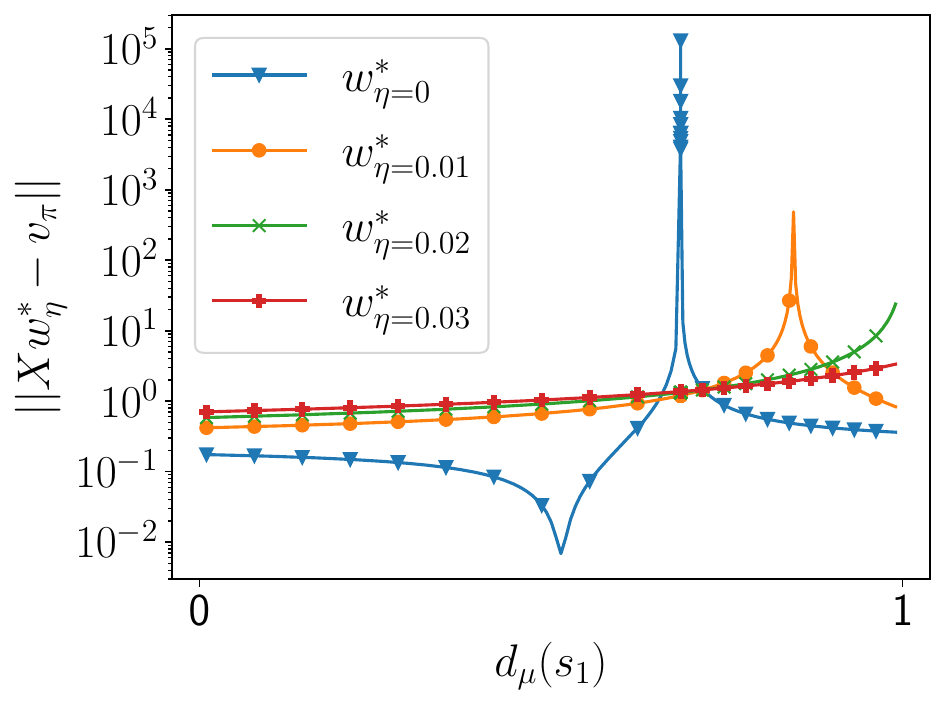}}
\subfloat[]{\includegraphics[width=0.25\linewidth]{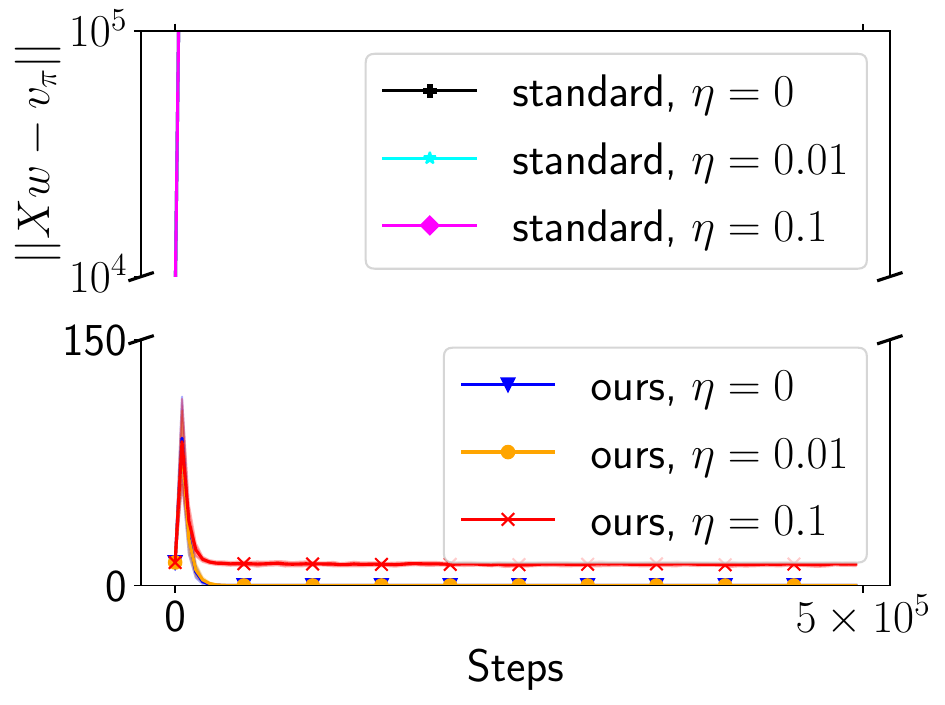}}
\subfloat[]{\includegraphics[width=0.25\linewidth]{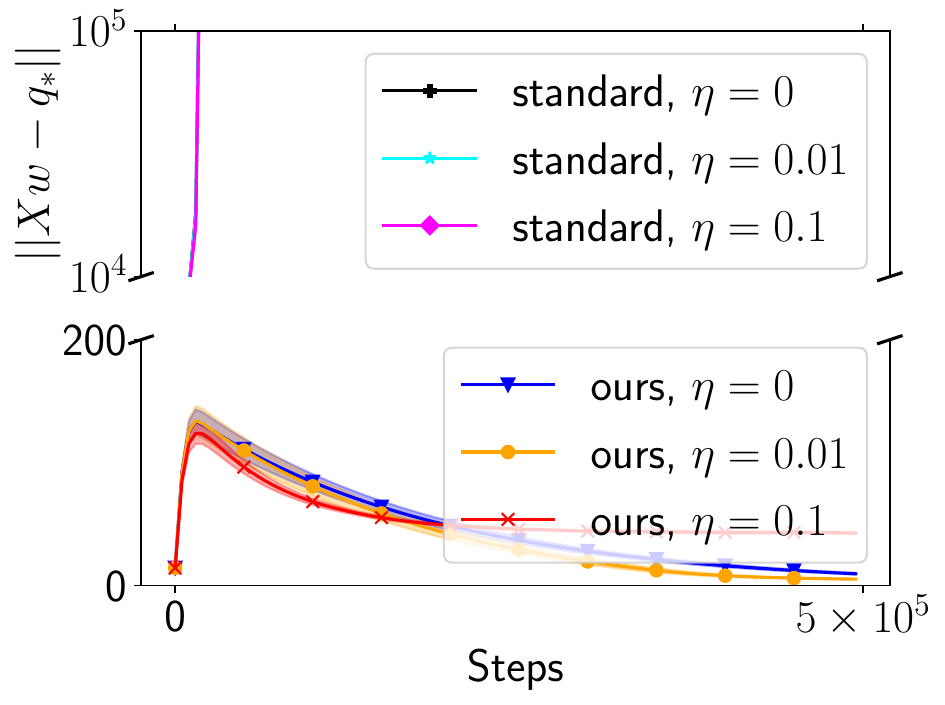}}
\subfloat[]{\includegraphics[width=0.25\linewidth]{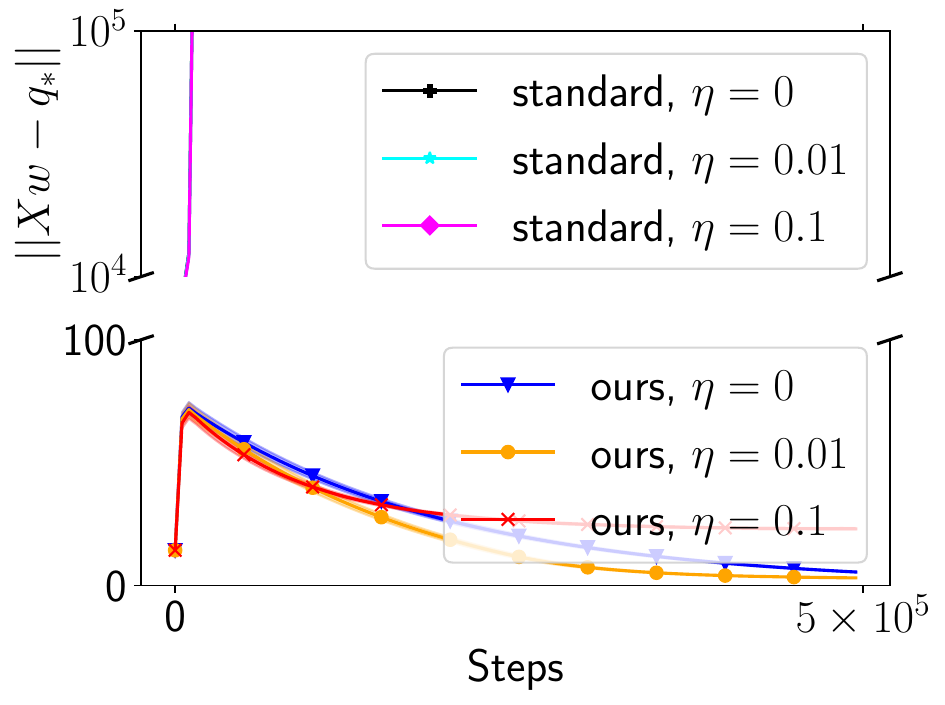}}
\caption{
\label{fig:expts}
(a) Effect of regularization on Kolter's example. 
$v_\pi$ is the true state-value function.
(b) Policy evaluation on Baird's example. 
(c) Control on Baird's example with a fixed behavior policy. 
(d) Control on Baird's example with an action-value-dependent behavior policy.
In (b)(c)(d),
the curves are averaged over 30 independent runs with shaded region indicating one standard deviation.
$q_*$ is the optimal action-value function.
$\eta$ is the weight for the ridge term.
Those marked ``ours'' are curves of algorithms we propose; those marked ``standard'' are standard semi-gradient off-policy algorithms.
Interestingly, 
the three ``standard'' curves overlap and get unbounded quickly.
}
\end{figure*}

\begin{theorem}
\label{thm:diff-linear-q}
Under Assumptions~\ref{assu full rank X},~\ref{assu target net lr},~\ref{assu target net all lrs},~\ref{assu:control-ergodic}, \&~\ref{assu:control-parameterization},
let $L_\mu$ denote the Lipschitz constant of $\mu_\theta$,
for any $\xi \in (0, 1), R_{B_1} > R_{B_2} > R_{B_1} - \xi > 0$,
there exist constants $C_0$ and $C_1$ such that for all 
% \begin{align}
$\norm{X} < C_0, L_\mu < C_1$,
% \end{align}
the iterate $\qty{w_t}$ generated by Algorithm~\ref{alg:diff-linear-q} satisfies 
\begin{align}
\textstyle{ \lim_{t \to \infty} w_t = w^*_\eta} \qq{almost surely,}
\end{align}
where $w^*_\eta$ is the unique solution of 
% \begin{align}
% \label{eq dff linear q obj}
\\$(\bar A_{\pi_w, \mu_w}+\eta I) w - \bar b_{\mu_w} = \tb{0}$
% \end{align}
inside $B_1$,
where
\begin{align}
\bar A_{\pi_w, \mu_w} &\doteq X(D_{\mu_w} - d_{\mu_w} d_{\mu_w}^\top) (I - P_{\pi_w}) X, \\
\bar b_{\mu_w} &\doteq X^\top (D_{\mu_w} - d_{\mu_w} d_{\mu_w}^\top)r,
\end{align}
and $\pi_w$ is a greedy policy w.r.t. $x(s, \cdot)^\top w$.
\end{theorem}
We defer the proof to Section~\ref{sec:proof-diff-linear-q}.
Theorem~\ref{thm:diff-linear-q} requires $\mu_\theta$ to be sufficiently smooth,
which is a standard assumption even in the on-policy setting (e.g., \citet{melo2008analysis,zou2019finite}).
It is easy to see that if~\eqref{eq diff Q update} converges,
it converges to a solution of
% \begin{align}
$\bar A_{\pi_w, \mu_w} w - \bar b_{\mu_w} = \tb{0}$,
% \end{align}
which we call a TD fixed point for control in the average-reward setting.
Theorem~\ref{thm:diff-linear-q}, which shows that Algorithm~\ref{alg:diff-linear-q} finds a regularized TD fixed point, is to the best of our knowledge the first theoretical study for linear $Q$-learning in the average-reward setting.

\section{Experiments}
\label{sec expts}
All the implementations are publicly available.~\footnote{\url{https://github.com/ShangtongZhang/DeepRL}}

We first use Kolter's example \citep{kolter2011fixed} to investigate how $\eta$ influences the performance of $w^*_\eta$ in the policy evaluation setting.
Details are provided in Section~\ref{sec kolter}.
This example is a two-state MDP with small representation error (i.e., $\norm{\Pi_{D_\mu} v_\pi - v_\pi}$ is small).
We vary the sampling probability of one state ($d_\mu(s_1)$) and compute corresponding $w^*_\eta$ analytically.
Figure~\ref{fig:expts}a shows that with $\eta = 0$,
the performance of $w^*_\eta$ becomes arbitrarily poor when $d_\mu(s_1)$ approaches around 0.71.
With $\eta = 0.01$, 
the spike exists as well.
If we further increase $\eta$ to $0.02$ and $0.03$,
the performance for $w^*_\eta$ becomes well bounded.
This confirms the potential advantage of the regularized TD fixed points.

We then use Baird's example \citep{baird1995residual} to empirically investigate the convergence of the algorithms we propose.
We use exactly the same setup as Chapter 11.2 of \citet{sutton2018reinforcement}.
Details are provided in Section~\ref{sec baird}.
In particular, we consider three settings:
policy evaluation (Figure~\ref{fig:expts}b),
control with a fixed behavior policy (Figure~\ref{fig:expts}c),
and control with an action-value dependent behavior policy (Figure~\ref{fig:expts}d).
For the policy evaluation setting,
we compare a TD version of Algorithm~\ref{alg:off-policy-td} and standard Off-Policy Linear TD (possibly with ridge regularization).
For the two control settings,
we compare Algorithm~\ref{alg:linear-q} with standard linear $Q$-learning (possibly with ridge regularization).
We use constant learning rates and do {not} use any projection in all the compared algorithms.
% \notey{This is Good! Actually pretty strong. Projection is identity. ``do {not} use any projection'' is not very precise.  }
The exact update rules are provided in Section~\ref{sec baird}.
Interestingly,
Figures~\ref{fig:expts}b-d show that even with $\eta = 0$, i.e., no ridge regularization,
our algorithms with target network still converge in the tested domains.
% This is expected as the conditions on $\norm{X}$ (or equivalently, $\eta$) are only sufficient and not necessarily necessary.
% \notey{``necessarily necessary'' reads a bit funny. }
By contrast,
without a target network,
even when mild regularization is imposed,
standard off-policy algorithms still diverge.
This confirms the importance of the target network.
% \notey{This experiment should be summarized in a general sense and motivation the work in introduction. }

\section{Discussion and Related Work}

For all the algorithms we propose,
both the target network and the ridge regularization are at play.
One may wonder if it is possible to ensure convergence with only ridge regularization without the target network. 
In the policy evaluation setting,
the answer is affirmative.
Applying ridge regularization in~\eqref{eq TD update} directly yields
\begin{align}
\label{eq td with ridge}
w_{t+1} \gets w_t + \alpha_t \delta_t x_t - \alpha_t \eta w_t,
\end{align}
where $\delta_t$ is defined in~\eqref{eq TD update}.
The expected update of~\eqref{eq td with ridge} is  
\begin{align}
\Delta_w &\doteq b - (A + \eta I)w \\
&\doteq b - X^\top D_\mu Xw + \gamma X^\top D_\mu (P_\pi Xw) - \eta w.
\end{align}
If its Jacobian w.r.t. $w$, denoted as $J_w(\Delta_w)$,
is negative definite,
the convergence of $\qty{w_t}$ is expected (see, e.g., Section 5.5 of \citet{vidyasagar2002nonlinear}).
This negative definiteness can be easily achieved by ensuring 
$\eta > \norm{X}^2 \norm{D_\mu(I-\gamma P_\pi)}$ (see \citet{diddigi2019convergent} for similar techniques).
% \notey{Is it easy to prove? can you explain intuitively?}
% Similar techniques are also used in .
% \notey{The bib file is not up to date. }
This direct ridge regularization, however,
would not work in the control setting.
Consider, for example, linear $Q$-learning with ridge regularization (i.e.,~\eqref{eq td with ridge} with $\delta_t$ defined in~\eqref{eq Q update}).
% \begin{align}
% \label{eq q with ridge}
% w_{t+1} &\gets w_t + \alpha_t \delta_t x_t - \alpha_t \eta w_t,
% \end{align}
% where $\delta_t$ is defined in~\eqref{eq Q update}. 
The Jacobian of its expected update is $J_w(b_{\mu_w} - (A_{\pi_w, \mu_w} + \eta I) w)$.
It is, however, not clear how to ensure this Jacobian is negative definite by tuning $\eta$.
By using a target network for bootstrapping, 
$P_\pi Xw$ becomes $P_\pi X\theta$.
So
$J_w(\Delta_w)$ becomes $-J_w(X^\top D_\mu Xw + \eta w)$,
which is always negative definite.
Similarly,
$J_w(b_{\mu_w} - (A_{\pi_w, \mu_w} + \eta ) w)$ becomes $-J_w(X^\top D_{\mu_\theta} X w + \eta w)$ in Algorithm~\ref{alg:linear-q},
which is always negative definite regardless of $\theta$.
% \notey{The explanation is a good try but not intuitive enough. It reads not easy to understand. }
The convergence of the main network $\qty{w_t}$ can, therefore, be expected. 
The convergence of the target network $\qty{\theta_t}$ is then delegated to Theorem~\ref{thm:target-net}.
Now it is clear that in the deadly triad setting,
the target network stabilizes training by ensuring the Jacobian of the expected update is negative definite.
One may also wonder if it is possible to ensure convergence with only the target network without ridge regularization.
The answer is unclear.
In our analysis,
the conditions on $\norm{X}$ (or equivalently, $\eta$) are only sufficient and not necessarily necessary.
We do see in Figure~\ref{fig:expts} that even with $\eta = 0$,
our algorithms still converge in the tested domains.
How small $\eta$ can be in general and under what circumstances $\eta$ can be 0 are still open problems,
which we leave for future work.
Further, ridge regularization usually affects the convergence rate of the algorithm,
which we also leave for future work.

In this paper,
we investigate target network as one possible solution for the deadly triad.
Other solutions include 
Gradient TD methods (\citet{sutton2009convergent,sutton2009fast,sutton2016emphatic} for the discounted setting; \citet{zhang2020average} for the average-reward setting) and Emphatic TD methods (\citet{sutton2016emphatic} for the discounted setting).
Other convergence results of $Q$-learning with function approximation include \citet{tsitsiklis1996feature,szepesvari2004interpolation},
which require special approximation architectures,
\citet{wen2013efficient,du2020agnostic},
which consider deterministic MDPs,
\citet{,li2011knows,du2019provably},
which require a special oracle to guide exploration,
\citet{chen2019zap},
which require matrix inversion every time step,
and \citet{wang2019optimism,yang2019sample,yang2020reinforcement,jin2020provably},
which consider linear MDPs (i.e., both $p$ and $r$ are assumed to be linear).
\citet{achiam2019towards} characterize the divergence of $Q$-learning with nonlinear function approximation via Taylor expansions and use preconditioning to empirically stabilize training.
\citet{van2018deep} empirically study the role of a target network in the deadly triad setting in deep RL,
which is complementary to our theoretical analysis.

Regularization is also widely used in RL.
\citet{yu2017convergence} introduce a general regularization term to improve the robustness of Gradient TD algorithms.
\citet{du2017stochastic} use ridge regularization in MSPBE to improve its convexity.
\citet{zhang2019provably} use ridge regularization to stabilize the training of critic in an off-policy actor-critic algorithm.
\citet{kolter2009regularization,johns2010linear,petrik2010feature,painter2012l1,liu2012regularized} use Lasso regularization in policy evaluation,
mainly for feature selection.

\section{Conclusion}
In this paper,
we proposed and analyzed a novel target network update rule,
with which we improved several linear RL algorithms that are known to diverge previously
due to the deadly triad.
Our analysis provided a theoretical understanding,
in the deadly triad setting,
of the conventional wisdom that a target network stabilizes training.
A possibility for future work is to introduce nonlinear function approximation,
possibly over-parameterized neural networks,
into our analysis.

\section*{Acknowledgments}
The authors thank Handong Lim for an insightful discussion.
SZ is generously funded by the Engineering and Physical Sciences Research Council (EPSRC).
SZ was also partly supported by DeepDrive. Inc from September to December 2020 during an internship. This project has received funding from the European Research Council under the European Union's Horizon 2020 research and innovation programme (grant agreement number 637713). The experiments were made possible by a generous equipment grant from NVIDIA. 

\bibliography{ref.bib}
\bibliographystyle{icml2021}

% \end{document}
\onecolumn
\newpage
\appendix

\section{Convergence of Target Networks}
We first state a result from \citet{borkar2009stochastic} regarding the convergence of a linear system. 
Consider updating the parameter $y \in \R^K$ recursively as
\begin{align}
\label{eq:borkar}
y_{t+1} \doteq y_t + \beta_t (h(y_t) + \epsilon_t),
\end{align} 
where $h: \R^K \to \R^K$ and $\qty{\epsilon_t}$ is a deterministic or random bounded sequence satisfying $\lim_{t \to \infty} \norm{\epsilon_t} = 0$. 
Assuming
\begin{assumption}
\label{assu:borkar-lipschitz}
$h$ is Lipschitz continuous.
\end{assumption}
\begin{assumption}
\label{assu:borkar-step}
The learning rates $\qty{\beta_t}$ satisfies $\sum_t \beta_t = \infty, \sum_t \beta_t^2 < \infty$.
\end{assumption}
\begin{assumption}
\label{assu:borkar-boundedness}
$\sup_t \norm{y_t} < \infty$ almost surely.
\end{assumption}
\begin{theorem}
\label{thm:borkar}
(The third extension of Theorem 2 in Chapter 2 of \citet{borkar2009stochastic}) Under Assumptions~\ref{assu:borkar-lipschitz}-~\ref{assu:borkar-boundedness},
almost surely,
the sequence $\qty{y_t}$ generated by~\eqref{eq:borkar} converges to a compact connected internally chain transitive invariant set of the ODE
\begin{align}
\dv{t} y(t) = h(y(t)).
\end{align}
\end{theorem}

\subsection{Proof of Theorem~\ref{thm:target-net}}
\label{sec convergence proof general target net}
\begin{proof}
Similar to Chapter 5.4 of \citet{borkar2009stochastic},
we consider $\dot{\Gamma}_{B_1}$, 
the directional derivative of $\Gamma_{B_1}$.
At a point $x \in \R^K$, given a direction $y \in R^K$,
we have
\begin{align}
\dot{\Gamma}_{B_1}(x, y) &\doteq \lim_{\delta \to 0} \frac{\Gamma_{B_1}(x + \delta y) - \Gamma_{B_1}(x)}{\delta} \\
& = \begin{cases}
y, & x \in int(B_1) \\
y, & x \in \partial B_1, y \in F_x(B_1) \\
-\frac{xx^\top y}{\norm{x}^3} + \frac{y}{\norm{x}}, &\text{otherwise}
\end{cases}
\end{align}
where $int(B_1)$ is the interior of $B_1$,
$\partial B_1$ is the boundary of $B_1$,
$F_x(B_1) \doteq \qty{y \in \R^K \mid \exists \delta > 0,\qq{s.t.}x + \delta y \in B_1}$ is the feasible directions of $B_1$ w.r.t. $x$.
The first two cases are trivial and are easy to deal with.
The third case is complicated and is the source of the reflection term $\zeta(t)$ in~\eqref{eq projected ode}.
However,
thanks to the projection $\Gamma_{B_2}$,
we succeeded in getting rid of it.
% the third case is not used in our proof. 
% We provide the analytical expression of the third case for completeness.
% By Algorithm~\ref{alg:general-target-net},
% we have $\norm{\theta_t} \leq R_{B_1}$ holds for all $t \geq 0$.

By~\eqref{eq general target net update}, $\theta_t \in B_1$ always holds.
With the directional derivative, 
we can rewrite the update rule of $\qty{\theta_t}$ as
\begin{align}
\theta_{t+1} &= \Gamma_{B_1} \big( \theta_t + \beta_t (\Gamma_{B_2}(w_t) - \theta_{t}) \big) \\
& = \theta_t + \beta_t \frac{\Gamma_{B_1} \big( \theta_t + \beta_t (\Gamma_{B_2}(w_t) - \theta_{t}) \big) - \Gamma_{B_1}(\theta_t)}{\beta_t} \\
& = \theta_t + \beta_t \big( \dot{\Gamma}_{B_1}(\theta_t, \Gamma_{B_2}(w_t) - \theta_{t}) + o(\beta_t) \big) \qq{(Definition of limit)}.
\end{align}
We now compute $\dot{\Gamma}_{B_1}(\theta_t, \Gamma_{B_2}(w_t) - \theta_{t})$.
% By~\eqref{eq general target net update},
% $\theta_t \in B_1$ holds for all $t$. 
We proceed by showing that only the first two cases in $\dot \Gamma_{B_1}(x, y)$ can happen and the third case will never occur.

For $\theta_t \in \text{int}(B_1)$,
we have 
\begin{align}
\label{eq:app-target-net-case1}
\dot{\Gamma}_{B_1}(\theta_t, \Gamma_{B_2}(w_t) - \theta_{t}) = \Gamma_{B_2}(w_t) - \theta_{t}.
\end{align} 
For $\theta_t \in \partial B_1$,
\begin{align}
\label{eq:app-target-net-inner-product}
\braket{\theta_t}{\Gamma_{B_2}(w_t) - \theta_t} &= \braket{\theta_t}{\Gamma_{B_2}(w_t)} - R_{B_1}^2 \leq R_{B_1} R_{B_2} - R_{B_1}^2 < 0.
\end{align}
Let $y_0 \doteq \Gamma_{B_2}(w_t) - \theta_t$,
\eqref{eq:app-target-net-inner-product} implies that we can decompose $y_0$ as $y_0 = y_1 + y_2$,
where $\braket{\theta_t}{y_1} = 0$ and $\braket{\theta_t}{y_2} = -\norm{\theta_t}\norm{y_2}$.
Here $y_2$ is the projection of $y_0$ onto $\theta_t$, which is in the opposite direction of $\theta_t$ and
$y_1$ is the remaining orthogonal component. 
% \notey{This sentence is a bit unclear. $\braket{\theta_t}{y_1} = 0$ and $\braket{\theta_t}{y_2} = -\norm{\theta_t}\norm{y_2}$: This basically says $\theta_t$ and $y_1$ are the same direction and $y_2$ is orthogonal to them. -- It's true whether or not implied by ~\eqref{eq:app-target-net-inner-product}.}
By Pythagoras's theorem, for any $\delta > 0$,
\begin{align}
\norm{\theta_t + \delta y_0}^2 &= \norm{\delta y_1}^2 + \norm{\theta_t + \delta y_2}^2 \\
% &= \delta^2 \norm{y_1}^2  + (\norm{\theta_t} - \delta \norm{y_2})^2 \\
&= \delta^2 \norm{y_1}^2 + \norm{\theta_t}^2 - 2\delta \norm{\theta_t}\norm{y_2} + \delta^2 \norm{y_2}^2.
\end{align}
For sufficiently small $\delta$, e.g., $\delta^2 \norm{y_1}^2 - 2\delta \norm{\theta_t}\norm{y_2} + \delta^2 \norm{y_2}^2 < 0$,
we have 
\begin{align}
\norm{\theta_t + \delta y_0}^2 < \norm{\theta_t}^2 = R_{B_1}^2,
\end{align}

implying $\Gamma_{B_2}(w_t) - \theta_t \in F_{\theta_t}(B_1)$.
So we have
\begin{align}
\label{eq:app-target-net-case2}
\dot{\Gamma}_{B_1}(\theta_t, \Gamma_{B_2}(w_t) - \theta_{t}) = \Gamma_{B_2}(w_t) - \theta_{t}.
\end{align} 
Combining~\eqref{eq:app-target-net-case1} and~\eqref{eq:app-target-net-case2} yields
\begin{align}
\theta_{t+1} &= \theta_t + \beta_t \big( \Gamma_{B_2}(w_t) - \theta_{t} + o(\beta_t) \big) \\
&=\theta_t + \beta_t \big( \Gamma_{B_2}(w^*(\theta_t)) - \theta_{t} + o(1) + o(\beta_t) \big) 
\qq{(Assumption~\ref{assu target net main net} and the continuity of $\Gamma_{B_2}$)} \\
&=\theta_t + \beta_t \big(w^*(\theta_t) - \theta_{t} + o(1) + o(\beta_t)) \qq{(Assumption~\ref{assu target net boundedness})}
\end{align}
Assumption~\ref{assu:borkar-lipschitz} is verified by Assumption~\ref{assu target net contraction};
Assumption~\ref{assu:borkar-step} is verified by Assumption~\ref{assu target net lr};
Assumption~\ref{assu:borkar-boundedness} is verified directly by the projection in~\eqref{eq general target net update}. 
By Theorem~\ref{thm:borkar},
almost surely,
$\qty{\theta_t}$ converges to a compact connected internally chain transitive invariant set of the ODE
\begin{align}
\dv{t} \theta(t) = w^*(\theta(t)) - \theta(t).
\end{align}
Under Assumption~\ref{assu target net contraction},
the Banach fixed-point theorem asserts that there is a unique $\theta^*$ satisfying $w^*(\theta^*) = \theta^*$,
i.e., $\theta^*$ is the unique equilibrium of the ODE above.
We now show $\theta^*$ is globally asymptotically stable.
Consider the candidate Lyapunov function
\begin{align}
V(\theta) \doteq \frac{1}{2}\norm{\theta - \theta^*}^2.
\end{align}
We have 
\begin{align}
\dv{t} V(\theta(t)) = &\braket{\theta(t) - \theta^*}{\dv{t} \theta(t)} \\
=&\braket{\theta(t) - \theta^*}{w^*(\theta(t)) - \theta(t)} \\
=&\braket{\theta(t) - \theta^*}{w^*(\theta(t)) - \theta(t) - w^*(\theta^*) + \theta^*} \\
\leq& \varrho \norm{\theta(t) - \theta^*}^2 - \norm{\theta(t) - \theta^*}^2,
\end{align}
where $\varrho < 1$ is the Lipschitz constant of $w^*$.
It is easy to see
\begin{itemize}
	\item $V(\theta) \geq 0$
	\item $V(\theta) = 0 \iff \theta = \theta^*$
	\item $\dv{t} V(\theta(t)) \leq 0$
	\item $\dv{t} V(\theta(t)) = 0 \iff \theta = \theta^*$
\end{itemize}
Consequently, 
$\theta^*$ is globally asymptotically stable,
implying
\begin{align}
\lim_{t \to \infty} \theta_t &= \theta^*, \\
\lim_{t \to \infty} w_t &= \lim_{t \to \infty} w^*(\theta_t) = w^*(\theta^*) = \theta^*. 
\end{align}
\end{proof}

\subsection{Proof of Lemma~\ref{lem target net changes slowly}}
\label{sec target net changes slowly}
\begin{proof}
By~\eqref{eq general target net update},
$\theta_t \in B_1$ holds for all $t$. 
So
\begin{align}
&\norm{\theta_{t+1} - \theta_t} \\
=& \norm{\Gamma_{B_1} \Big( \theta_t + \beta_t (\Gamma_{B_2}(w_t) - \theta_{t}) \Big) - \Gamma_{B_1}(\theta_t)} \\
\leq& \beta_t \norm{\Gamma_{B_2}(w_t) - \theta_{t}} \qq{(Nonexpansiveness of projection)} \\
\leq& \beta_t (R_{B_1} + R_{B_2}).
\end{align}
\end{proof}

\subsection{Proof of Theorem~\ref{thm:op-td}}
\label{sec:proof-of-td}
\begin{proof}
Consider the Markov process $Y_t \doteq (S_t, A_t, S_{t+1})$. 
% evolving in the space $\Y \doteq \mathcal{S} \times \mathcal{A} \times \mathcal{S}$.
By Assumption~\ref{assu:op-td-ergodic},
$Y_t$ adopts a unique stationary distribution,
which we refer to as $d_\Y$.
We have $d_\Y(s, a, s') = d_\mu(s) \mu(a|s) p(s'|s, a)$.
We define
\begin{align}
\label{eq:app-of-td-reward}
h_\theta(s, a, s') &\doteq \big( r(s, a) + \gamma \sum_{a'} \pi(a'|s') x(s', a')^\top \Gamma_{B_1}(\theta) \big) x(s, a), \\
G_\theta(s, a, s') &\doteq x(s, a) x(s, a)^\top + \eta I.
\end{align}
As $\theta_t \in B_1$ holds for all $t$,
we can rewrite the update of $w_t$ in Algorithm~\ref{alg:off-policy-td} as
\begin{align}
w_{t+1} = w_t + \alpha_t (h_{\theta_t}(Y_t) - G_{\theta_t}(Y_t)w_t).
\end{align}
The asymptotic behavior of $\qty{w_t}$ is then governed by
\begin{align}
\label{eq:app-of-td-gain}
\bar{h}(\theta) &\doteq \E_{(s, a, s')\sim d_\Y}[h_\theta(s, a, s')] \\
& = X^\top D_\mu r + \gamma X^\top D_\mu P_\pi X \Gamma_{B_1}(\theta), \\
\bar{G}(\theta) &\doteq \E_{(s, a, s')\sim d_\Y}[G_\theta(s, a, s')] \\
& = X^\top D_\mu X + \eta I.
\end{align}
Define
\begin{align}
\label{eq:app-of-td-1}
w^*(\theta) \doteq \bar{G}(\theta)^{-1}\bar{h}(\theta)
= (X^\top D_\mu X + \eta I)^{-1} X^\top D_\mu (r + \gamma P_\pi X \Gamma_{B_1}(\theta)).
\end{align}
We now verify Assumptions~\ref{assu target net main net},~\ref{assu target net boundedness}, \&~\ref{assu target net contraction} to invoke Theorem~\ref{thm:target-net}.

Assumption~\ref{assu target net main net} is verified in Lemma~\ref{lem:main-of-td}.

To verify Assumption~\ref{assu target net contraction},
we use SVD and get 
\begin{align}
D_\mu^\frac{1}{2}X = U^\top \Sigma V,
\end{align}
where $U, V$ are two orthogonal matrices, $\Sigma = \mqty[\Sigma_+ \\ 0]$ is a rectangular diagonal matrix with 
$\Sigma_+ \doteq diag([\dots, \sigma_i, \dots])$ being a diagonal matrix.
Assumptions~\ref{assu:op-td-ergodic} \& \ref{assu full rank X} imply that $\sigma_i > 0$.
We have 
\begin{align}
\label{eq:op-td-contraction}
\norm{(X^\top D_\mu X + \eta I)^{-1} X^\top D_\mu^\frac{1}{2}} =& \norm{V^\top (\Sigma^\top \Sigma + \eta I)^{-1} \Sigma^\top U } 
= \norm{(\Sigma^\top \Sigma + \eta I)^{-1} \Sigma^\top} \\ 
=& \norm{diag([\cdots, \frac{\sigma_i}{\sigma_i^2 + \eta}, \cdots])}
= \norm{diag([\cdots, \frac{1}{\sigma_i + \eta / \sigma_i}, \cdots])}
= \max_i \frac{1}{\sigma_i + \eta / \sigma_i} \\
\leq& \frac{1}{2\sqrt{\eta}} 
\end{align}
% \notey{$\norm{V} = \norm{U} = 1$ is not correct, but this line still holds for 2-norm because the any vector is invariant under the multiplication of an orthogonal matrix. }
% \notey{add before the last equality, $=\max_i   \frac{1}{\sigma_i + \eta / \sigma_i}$}
According to~\eqref{eq:app-of-td-1}, it is then easy to see
\begin{align}
\norm{w^*(\theta_1) - w^*(\theta_2)} &\leq \frac{\gamma}{2\sqrt{\eta}} \norm{D_\mu^\frac{1}{2} P_\pi} \norm{X} \norm{\Gamma_{B_1}(\theta_1) - \Gamma_{B_1}(\theta_2)} \\
&\leq \frac{\gamma}{2\sqrt{\eta}} \norm{D_\mu^\frac{1}{2} P_\pi} \norm{X} \norm{\theta_1 - \theta_2} \\
&\leq \frac{\gamma}{2\sqrt{\eta}} \norm{D_\mu^\frac{1}{2} P_\pi D_\mu^{-\frac{1}{2}}} \norm{D_\mu^{\frac{1}{2}}} \norm{X} \norm{\theta_1 - \theta_2} \\
&= \frac{\gamma}{2\sqrt{\eta}} \norm{P_\pi}_{D_\mu} \norm{D_\mu^{\frac{1}{2}}} \norm{X} \norm{\theta_1 - \theta_2} \qq{(See (20) in \citet{kolter2011fixed})}\\
&\leq \frac{\gamma}{2\sqrt{\eta}} \norm{P_\pi}_{D_\mu} \norm{X} \norm{\theta_1 - \theta_2}. \\
\end{align}
Take any $\xi \in (0, 1)$, assuming
\begin{align}
\label{eq:op-td-inequ1}
\norm{X} \leq \frac{2(1 - \xi) \sqrt{\eta}}{\gamma \norm{P_\pi}_{D_\mu}},
\end{align}
% \notey{Change to 
% \[
% \norm{X} \leq \frac{2(1 - \xi) \sqrt{\eta}}{\gamma \norm{P_\pi}_{D_\mu}} = C_0,
% \]
% }

% \notey{Assumption on the features...I think it's better to put the condition on $\eta$, which is 
% $\eta \ge \left(\gamma \frac{\norm{X} \norm{P_\pi}_{D_\mu}}{2(1-\xi)}\right)^2$.
% So when $1-\xi$ is close to zero, $\eta$ is very large. 
% A very large $\eta$ may make the solution doesn't make much sense. 
% Can we have some understanding on the how small $\eta$ can be? 
% We need meaningful solutions for off-policy evaluation. 
% A very large $\eta$ make the solution deviate from the TD solution too much.  
% }
% \notey{Update: actually would it be a good idea to set $\eta$ to the minimum value? $\eta= \left(\gamma \frac{\norm{X} \norm{P_\pi}_{D_\mu}}{2(1-\xi)}\right)^2$. }
then
\begin{align}
\norm{w^*(\theta_1) - w^*(\theta_2)} \leq (1 - \xi)\norm{\theta_1 - \theta_2}.
\end{align}
Assumption~\ref{assu target net contraction}, therefore, holds.

We now select proper $R_{B_1}$ and $R_{B_2}$ to fulfill Assumption~\ref{assu target net boundedness}.
Plugging \eqref{eq:op-td-contraction} and \eqref{eq:op-td-inequ1} into~\eqref{eq:app-of-td-1} yields 
\begin{align}
\norm{w^*(\theta)} &\leq \frac{1}{2\sqrt{\eta}} \norm{D_\mu^\frac{1}{2}r} + (1 - \xi) R_{B_1} \\
&= R_{B_1} - \xi + (\frac{1}{2\sqrt{\eta}}\norm{D_\mu^\frac{1}{2}r} + \xi - \xi R_{B_1})
\end{align}
For sufficiently large $R_{B_1}$, e.g.,
\begin{align}
\label{eq:op-td-inequ2}
R_{B_1} \geq \frac{1}{2\xi\sqrt{\eta}}\norm{D_\mu^\frac{1}{2}r} + 1,
\end{align}
we have $\sup_\theta \norm{w^*(\theta)} \leq R_{B_1} - \xi$.
Selecting $R_{B_2} \in (R_{B_1} - \xi, R_{B_1})$ then fulfills Assumption~\ref{assu target net boundedness}.

With Assumptions~\ref{assu target net lr} -~\ref{assu target net contraction} satisfied, Theorem~\ref{thm:target-net} then implies that there exists a unique $\theta_\infty$ such that
\begin{align}
w^*(\theta_\infty) = \theta_\infty \qq{and} \lim_{t\to \infty} \theta_t = \lim_{t \to \infty} w_t = \theta_\infty. 
\end{align}

Next we show what $\theta_\infty$ is.
We define
\begin{align}
f(\theta) \doteq (X^\top D_\mu X + \eta I)^{-1} X^\top D_\mu (r + \gamma P_\pi X \theta).
\end{align}
Note this is just the right side of equation \eqref{eq:app-of-td-1} without the projection.
\eqref{eq:op-td-contraction} and \eqref{eq:op-td-inequ1} imply that $f$ is a contraction.
The Banach fixed-point theorem then asserts that $f$ adopts a unique fixed point,
which we refer to as $w^*_\eta$.
We have
\begin{align}
\norm{w^*_\eta} &= \norm{f(w^*_\eta)} \leq \frac{1}{2\sqrt{\eta}} (\norm{r} + \gamma \norm{P_\pi}_{D_\mu} \norm{X} \norm{w^*_\eta}) \\
&\leq \frac{\norm{r}}{2 \sqrt{\eta}} + (1 - \xi) \norm{w^*_\eta} \qq{(By~\eqref{eq:op-td-inequ1})} \\
\implies \norm{w^*_\eta} &\leq \frac{\norm{r}}{2 \xi \sqrt{\eta}}
\end{align}
Then for sufficiently large $R_{B_1}$, e.g.,
\begin{align}
\label{eq:app-op-td-ineq3}
R_{B_1} \geq \frac{\norm{r}}{2 \xi \sqrt{\eta}},
\end{align}

% \notey{In equation \eqref{eq:op-td-inequ2}, it is $
% R_{B_1} \ge \frac{1}{2\xi\sqrt{\eta}}\norm{D_\mu^\frac{1}{2}r_\pi} + 1$, 
% should the max between the two taken?
% }
we have $w^*_\eta = \Gamma_{B_1}(w^*_\eta)$, implying 
$w^*_\eta$ is a fixed point of $w^*(\cdot)$ (i.e., the right side of \eqref{eq:app-of-td-1}) as well.
As $w^*(\cdot)$ is a contraction, we have $\theta_\infty = w^*_\eta$.
Rewriting $f(w^*_\eta) = w^*_\eta$ yields
\begin{align}
Aw^*_\eta + \eta w^*_\eta - b &= \tb{0}.
\end{align}
In other words, $w^*_\eta$ is the unique (due to the contraction of $f$) solution of $(A + \eta I)w - b = \tb{0}$.

Combining~\eqref{eq:op-td-inequ1}, ~\eqref{eq:op-td-inequ2}, and~\eqref{eq:app-op-td-ineq3},
the desired constants are
\begin{align}
C_0 &\doteq \frac{2(1 - \xi)\sqrt{\eta}}{\gamma \norm{P_\pi}_{D_\mu}}, \\
C_1 &\doteq \frac{\norm{r}}{2\xi\sqrt{\eta}} + 1.
\end{align}
We now bound $\norm{Xw^*_\eta - q_\pi}$.
For any $y \in \R^\ns$, we define the ridge regularized projection $\Pi^\eta_{D_\mu}$ as
\begin{align}
\Pi^\eta_{D_\mu} y &\doteq X \arg\min_w \left(\norm{Xw - y}^2_{D_\mu} + \eta \norm{w}^2 \right) \\
 &= X (X^\top D_\mu X + \eta I)^{-1} X^\top D_\mu y.
\end{align}
$\Pi_{D_\mu}^\eta$ is connected with $f$ as
$\Pi_{D_\mu}^\eta \vop_\pi (Xw) = Xf(w)$.
We have
\begin{align}
\norm{X w^*_\eta - q_\pi} &\leq \norm{X w^*_\eta - \Pi^\eta_{D_\mu} q_\pi} + \norm{\Pi^\eta_{D_\mu} q_\pi - q_\pi} \\
&=\norm{X f(w^*_\eta) - \Pi^\eta_{D_\mu} q_\pi} + \norm{\Pi^\eta_{D_\mu} q_\pi - q_\pi} \\
&=\norm{\Pi^\eta_{D_\mu}\vop_\pi (Xw^*_\eta) - \Pi^\eta_{D_\mu} \vop_\pi q_\pi} + \norm{\Pi^\eta_{D_\mu} q_\pi - q_\pi} \\
&=\norm{\Pi^\eta_{D_\mu}\gamma P_\pi Xw^*_\eta - \Pi^\eta_{D_\mu} \gamma P_\pi q_\pi} + \norm{\Pi^\eta_{D_\mu} q_\pi - q_\pi} \\
&=\norm{\gamma\Pi^\eta_{D_\mu} P_\pi (Xw^*_\eta - q_\pi)} + \norm{\Pi^\eta_{D_\mu} q_\pi - q_\pi} \\
&\leq \norm{\gamma X (X^\top D_\mu X + \eta I)^{-1} X^\top D_\mu P_\pi} \norm{Xw^*_\eta - q_\pi} + \norm{\Pi^\eta_{D_\mu} q_\pi - q_\pi} \\
&\leq (1 - \xi) \norm{Xw^*_\eta - q_\pi} + \norm{\Pi^\eta_{D_\mu} q_\pi - q_\pi} \qq{(By \eqref{eq:op-td-contraction} and \eqref{eq:op-td-inequ1})} \\
\end{align}
The above equation implies
\begin{align}
\norm{X w^*_\eta - q_\pi} &\leq \frac{1}{\xi} \norm{\Pi^\eta_{D_\mu} q_\pi - q_\pi} \leq \frac{1}{\xi} \left(\norm{\Pi^\eta_{D_\mu} q_\pi - \Pi_{D_\mu} q_\pi} + \norm{\Pi_{D_\mu} q_\pi - q_\pi}\right),
\end{align}
where $\Pi_{D_\mu}$ is shorthand for $\Pi_{D_\mu}^{\eta = 0}$.
% \notey{The third line below the connection of $f$ and $\Pi$ was not introduced.}
% \notey{$v_\pi$ is the fixed point? This needs to be explained. }
% \notey{Was $\Pi_{D_\mu}$ defined? Or just say it simplifies for $\Pi^0_{D_\mu}$.}
% \notey{the line with $1-\xi$ line looks like we can derive a lower bound for this contraction, $1-\xi$. 
% It appears that this $\xi$ is related to the contraction factor of $f(\theta)$. 
% I haven't figured it out exactly but it's interesting to have a look. 
% }
% \notey{Can you break this long equations into a few ones. It's a bit hard to find a line in tex. }
We now bound $\norm{\Pi^\eta_{D_\mu} - \Pi_{D_\mu}}$.
\begin{align}
\label{eq:regularized projection bound}
\norm{\Pi^\eta_{D_\mu} - \Pi_{D_\mu}} &\leq \norm{X}\norm{(X^\top D_\mu X + \eta I)^{-1} - (X^\top D_\mu X)^{-1}} \norm{X^\top D_\mu} \\
&= \norm{D_\mu^{-\frac{1}{2}} D_\mu^{\frac{1}{2}} X} \norm{V^\top \big( (\Sigma_+^2 + \eta I)^{-1} - \Sigma_+^{-2} \big) V} \norm{X^\top D_\mu^{\frac{1}{2}} D_\mu^{\frac{1}{2}}} \\
&\leq \norm{D_\mu^{-\frac{1}{2}}} \norm{\Sigma} \norm{(\Sigma_+^2 + \eta I)^{-1} - \Sigma_+^{-2}} \norm{\Sigma} \norm{D_\mu^{\frac{1}{2}}} \\
&\leq \norm{D_\mu^{-\frac{1}{2}}} \norm{\Sigma} \norm{diag([\dots, \frac{\eta}{\sigma_i^2(\eta + \sigma_i^2)} ,\dots])} \norm{\Sigma} \norm{D_\mu^{\frac{1}{2}}} \\
&\leq \norm{D_\mu^{-\frac{1}{2}}} \norm{\Sigma} \max_i\frac{\eta}{\sigma_i^2(\eta + \sigma_i^2)} \norm{\Sigma} \norm{D_\mu^{\frac{1}{2}}} \\
% &\leq \norm{D_\mu^{-\frac{1}{2}}} \norm{\Sigma} \max_i\frac{1}{\sigma_i^2} \norm{\Sigma} \norm{D_\mu^{\frac{1}{2}}} \\
&\leq \norm{D_\mu^{-\frac{1}{2}}} \norm{\Sigma}^2 \max_i\frac{\eta}{\sigma_i^4}  \norm{D_\mu^{\frac{1}{2}}} \\
&\leq \norm{D_\mu^{-\frac{1}{2}}} \frac{\sigma_{\max}(\Sigma)^2}{\sigma_{\min}(\Sigma)^4}  \norm{D_\mu^{\frac{1}{2}}} \eta
\intertext{\hfill ($\sigma_{\max}(\cdot)$ and $\sigma_{\min}(\cdot)$ indicate the largest and smallest singular values)}
&= \frac{\sigma_{\max}(D_\mu^{\frac{1}{2}}X)^2}{\sigma_{\min}(D_\mu^{\frac{1}{2}}X)^4} \frac{\sigma_{\max}(D_\mu^{\frac{1}{2}})}{\sigma_{\min}(D_\mu^{\frac{1}{2}})} \eta \\
% &\leq \bar{\kappa}(X)^2 \bar{\kappa}(D_\mu)^{\frac{3}{2}}
&\leq \frac{\sigma_{\max}(X)^2}{\sigma_{\min}(X)^4 \sigma_{\min}(D_\mu)^{2.5}} \eta 
\intertext{\hfill ($\sigma_{\max}(XY) \leq \sigma_{\max}(X)\sigma_{\max}(Y); \sigma_{\min}(XY) \geq \sigma_{\min}(X)\sigma_{\min}(Y)$ )}.
\end{align}
Finally, we arrive at
\begin{align}
\norm{X w^*_\eta - q_\pi} &\leq \frac{1}{\xi} \Big( \frac{\sigma_{\max}(X)^2}{\sigma_{\min}(X)^4 \sigma_{\min}(D_\mu)^{2.5}} \norm{q_\pi} \eta + \norm{\Pi_{D_\mu} q_\pi - q_\pi} \Big),
\end{align}
which completes the proof.
% \notey{This bound format suggests we can have a better/specific form for $\xi$, which should be related to the problem parameters, $P, D, X, \gamma$. Usually  $\xi \sim 1/(1-\gamma)$.}
% \notey{Can the right side be made to
% \[
% \frac{1}{\xi} \inf_v \Big( \bar{\kappa}(X)^2 \bar{\kappa}(D_\mu)^{\frac{3}{2}} \norm{v} + \norm{v - v_\pi} \Big),
% \]
% }
\end{proof}

\subsection{Proof of Theorem~\ref{thm:diff-op-td}}
\label{sec:proof-diff-of-td}
\begin{proof}
The proof is similar to the proof of Theorem~\ref{thm:op-td} in Section~\ref{sec:proof-of-td}.
We, therefore, highlight only the difference to avoid verbatim repetition.
Define 
\begin{align}
\theta &\doteq \mqty[\theta^r \\ \theta^w], u \doteq \mqty[\bar r \\ w], \\
h_\theta(s, a, s') &\doteq \mqty[ r(s, a) \\ x(s, a) r(s, a)] + \mqty[0 & \sum_{a'} \pi(a'|s')x(s', a')^\top - x(s, a)^\top \\
-x(s, a) & x(s, a) \sum_{a'} \pi(a'|s')x(s', a')^\top] \Gamma_{B_1}(\theta), \\
G_\theta(s, a, s') &\doteq \mqty[1 & \tb{0}^\top \\ \tb{0} & x(s, a) x(s, a)^\top + \eta I],
\end{align}
We can then rewrite the update of $\bar r$ and $w$ in Algorithm~\ref{alg:differential-off-policy-td} as
\begin{align}
u_{t+1} = u_t + \alpha_t (h_{\theta_t}(Y_t) - G_{\theta_t}(Y_t) u_t).
\end{align}
Similarly,
we define
\begin{align}
\bar{h}(\theta) &\doteq \E_{(s, a, s') \sim d_\Y}[h_\theta(s, a, s')] = \bar h_1 + \bar H_2  \Gamma_{B_1}(\theta), \\
\bar h_1 &\doteq \mqty[d_\mu^\top r \\ X^\top D_\mu r], \bar H_2 \doteq \mqty[0 & d_\mu^\top (P_\pi - I) X \\ -X^\top d_\mu & X^\top D_\mu P_\pi X] \\
\bar{G}(\theta) &\doteq \E_{(s, a, s') \sim d_\Y}[G_\theta(s, a, s')] = \mqty[1 & \tb{0}^\top \\ \tb{0} & X^\top D_\mu X + \eta I], \\
u^*(\theta) &\doteq \bar{G}(\theta)^{-1} \bar h(\theta).
\end{align}

We proceed to verifying Assumptions~\ref{assu target net main net},~\ref{assu target net boundedness}, \&~\ref{assu target net contraction} to invoke Theorem~\ref{thm:target-net}.

Assumption~\ref{assu target net main net} is verified in Lemma~\ref{lem:main-diff-of-td}.

For Assumption~\ref{assu target net contraction} to hold,
note
\begin{align}
\norm{\bar G(\theta)^{-1}} &= \max\qty{1, \norm{(X^\top D_\mu X + \eta I)^{-1}}} \leq \max \qty{1, \frac{1}{\eta}} \\
\norm{\bar H_2}^2 &= \max_{\norm{u} = 1} \norm{\bar H_2 u}^2 = \max_{\norm{u} = 1} \norm{\mqty[d_\mu^\top(P_\pi - I)Xw \\ - X^\top d_\mu \bar r + X^\top D_\mu P_\pi X w]}^2 \\
&= \max_{\norm{u} = 1} \norm{d_\mu^\top(P_\pi - I)Xw}^2 + \norm{- X^\top d_\mu \bar r + X^\top D_\mu P_\pi X w}^2 \\
&\leq \max_{\norm{u} = 1} \norm{d_\mu^\top(P_\pi - I)X}^2 \norm{w}^2 + 2 \norm{X^\top d_\mu}^2 \norm{\bar r}^2 + 2 \norm{X^\top D_\mu P_\pi X}^2 \norm{w}^2 \\
&\leq \max \qty{\norm{d_\mu^\top(P_\pi - I)X}^2 + 2 \norm{X^\top D_\mu P_\pi X}^2, 2 \norm{X^\top d_\mu}^2} \\
\implies \norm{\bar H_2} &\leq \norm{X} \max \qty{\norm{d_\mu^\top(P_\pi - I)} + \sqrt{2} \norm{D_\mu P_\pi}, \sqrt{2} \norm{d_\mu}}.
\end{align}
The above equations suggest that
\begin{align}
\norm{u^*(\theta_1) - u^*(\theta_2)} &\leq \max \qty{1, \frac{1}{\eta}} \norm{X} \max \qty{\norm{d_\mu^\top(P_\pi - I)} + \sqrt{2} \norm{D_\mu P_\pi}, \sqrt{2} \norm{d_\mu}} \norm{\theta_1 - \theta_2}.
\end{align}
Take any $\xi \in (0, 1)$, assuming
\begin{align}
\label{eq:diff-op-td-inequ1}
\norm{X} \leq \frac{1 - \xi}{\max \qty{1, \frac{1}{\eta}} \max \qty{\norm{d_\mu^\top(P_\pi - I)} + \sqrt{2} \norm{D_\mu P_\pi}, \sqrt{2} \norm{d_\mu}}},
\end{align}
then
\begin{align}
\norm{u^*(\theta_1) - u^*(\theta_2)} \leq (1 - \xi)\norm{\theta_1 - \theta_2}.
\end{align}
Assumption~\ref{assu target net contraction}, therefore, holds.

We now select proper $R_{B_1}$ and $R_{B_2}$ to fulfill Assumption~\ref{assu target net boundedness}.
Using~\eqref{eq:diff-op-td-inequ1},
it is easy to see
\begin{align}
\norm{u^*(\theta)} \leq \max \qty{1, \frac{1}{\eta}} \norm{\bar h_1} + (1 - \xi) R_{B_1}.
\end{align}
For sufficiently large $R_{B_1}$, e.g.,
\begin{align}
\label{eq:diff-op-td-inequ2}
R_{B_1} \geq \max \qty{1, \frac{1}{\eta}} \frac{\norm{\bar h_1}}{\xi} + 1,
\end{align}
we have $\sup_\theta \norm{w^*(\theta)} \leq R_{B_1} - \xi$.
Selecting $R_{B_2} \in (R_{B_1} - \xi, R_{B_1})$ then fulfills Assumption~\ref{assu target net boundedness}.

With Assumptions~\ref{assu target net lr} -~\ref{assu target net contraction} satisfied, Theorem~\ref{thm:target-net} then implies that there exists a unique $\theta_\infty$ such that
\begin{align}
u^*(\theta_\infty) = \theta_\infty \qq{and} \lim_{t\to \infty} \theta_t = \lim_{t \to \infty} u_t = \theta_\infty. 
\end{align}

Next we show what $\theta_\infty$ is.
We define
\begin{align}
f(\theta) \doteq \bar G(\theta)^{-1}(\bar h_1 + \bar H_2 \theta).
\end{align}
Note this is just $u^*(\theta)$ without the projection.
Under~\eqref{eq:diff-op-td-inequ1},
it is easy to show
$f$ is a contraction.
The Banach fixed-point theorem then asserts that $f$ adopts a unique fixed point,
which we refer to as $u^*_\eta$.
Using~\eqref{eq:diff-op-td-inequ1} again,
we get
\begin{align}
\norm{u^*_\eta} &= \norm{f(u^*_\eta)} \leq \max \qty{1, \frac{1}{\eta}} \norm{\bar h_1} + (1 - \xi) \norm{w^*_\eta}\\
\implies \norm{u^*_\eta} &\leq \max \qty{1, \frac{1}{\eta}} \frac{\norm{\bar h_1}}{\xi} 
\end{align}
Then for sufficiently large $R_{B_1}$, e.g.,
\begin{align}
\label{eq:app-diff-op-td-ineq3}
R_{B_1} \geq \max \qty{1, \frac{1}{\eta}} \frac{\norm{\bar h_1}}{\xi},
\end{align}
we have $u^*_\eta = \Gamma_{B_1}(u^*_\eta)$, implying 
$u^*_\eta$ is a fixed point of $u^*(\cdot)$ as well.
As $u^*(\cdot)$ is a contraction, we have $\theta_\infty = u^*_\eta$.
Writing $u^*_\eta$ as $\mqty[\bar r^*_\eta \\ w^*_\eta]$
and expanding $f(u^*_\eta) = u^*_\eta$ yields
\begin{align}
\bar r^*_\eta &= d_\mu^\top (r + P_\pi X w^*_\eta - X w^*_\eta), \\
(X^\top D_\mu X + \eta I) w^*_\eta &= X^\top D_\mu r - X^\top d_\mu \bar r^*_\eta + X^\top D_\mu P_\pi X w^*_\eta.
\end{align}
Rearranging terms yields $(\bar A + \eta I) w^*_\eta - \bar b = 0$, i.e., 
$w^*_\eta$ is the unique (due to the contraction of $f$) solution of $(\bar A + \eta I)w - \bar b = 0$.
% \notey{$\bar A, \bar b$ were not defined. }

We now bound $\norm{Xw^*_\eta - \bar q_\pi^c}$.
\begin{align}
\norm{X w^*_\eta - \bar q_\pi^c} &\leq \norm{X w^*_\eta - \Pi^\eta_{D_\mu} \bar q_\pi^c} + \norm{\Pi^\eta_{D_\mu} \bar q_\pi^c - \bar q_\pi^c} \\
&=\norm{X (X^\top D_\mu X + \eta I)^{-1}(X^\top D_\mu r - X^\top d_\mu \bar r^*_\eta + X^\top D_\mu P_\pi X w^*_\eta)  - \Pi^\eta_{D_\mu} \bar q_\pi^c} + \norm{\Pi^\eta_{D_\mu} \bar q_\pi^c - \bar q_\pi^c} \\
&=\norm{\Pi_{D_\mu}^\eta (r + P_\pi Xw^*_\eta) - \Pi_{D_\mu}^\eta (r + P_\pi \bar q_\pi^c)} + \norm{\Pi^\eta_{D_\mu} \bar q_\pi^c - \bar q_\pi^c} \qq{(Using $X^\top d_\mu = \tb{0}$)}\\
&=\norm{\Pi^\eta_{D_\mu} P_\pi (Xw^*_\eta - \bar q_\pi^c)} + \norm{\Pi^\eta_{D_\mu} \bar q_\pi^c - \bar q_\pi^c} \\
&\leq \norm{X (X^\top D_\mu X + \eta I)^{-1} X^\top D_\mu P_\pi} \norm{Xw^*_\eta - \bar q_\pi^c} + \norm{\Pi^\eta_{D_\mu} \bar q_\pi^c - \bar q_\pi^c} \\
&\leq \frac{1}{\eta} \norm{X}^2 \norm{D_\mu P_\pi} \norm{Xw^*_\eta - \bar q_\pi^c} + \norm{\Pi^\eta_{D_\mu} \bar q_\pi^c - \bar q_\pi^c}.
\end{align}
Assuming
\begin{align}
\label{eq:app-diff-op-td-ineq4}
\norm{X}^2 \leq \frac{(1 - \xi) \eta}{\norm{D_\mu P_\pi}},
\end{align}
we have
\begin{align}
\norm{X w^*_\eta - \bar q_\pi^c} \leq \frac{1}{\xi} \norm{\Pi^\eta_{D_\mu} \bar q_\pi^c - \bar q_\pi^c} \leq \frac{1}{\xi} \Big( \frac{\sigma_{\max}(X)^2}{\sigma_{\min}(X)^4 \sigma_{\min}(D_\mu)^{2.5}} \norm{\bar q_\pi^c} \eta + \norm{\Pi_{D_\mu} \bar q_\pi^c - \bar q_\pi^c} \Big) \qq{(c.f.~\eqref{eq:regularized projection bound})}.
\end{align}
It is then easy to see
\begin{align}
|\bar r^*_\eta - \bar r_\pi| \leq \norm{d_\mu^\top (P_\pi - I) (Xw^*_\eta - \bar q_\pi^c)} \leq \frac{1}{\xi} \norm{d_\mu^\top (P_\pi - I)}\Big( \frac{\sigma_{\max}(X)^2}{\sigma_{\min}(X)^4 \sigma_{\min}(D_\mu)^{2.5}} \norm{\bar q_\pi^c} \eta + \norm{\Pi_{D_\mu} \bar q_\pi^c - \bar q_\pi^c} \Big).
\end{align}
Taking infimum for $c \in \R$ then yields the desired results.

Combining~\eqref{eq:diff-op-td-inequ1},~\eqref{eq:diff-op-td-inequ2},~\eqref{eq:app-diff-op-td-ineq3}, and~\eqref{eq:app-diff-op-td-ineq4},
the desired constants are
\begin{align}
C_0 &\doteq \min \qty{ \frac{1 - \xi}{\max \qty{1, \frac{1}{\eta}} \max \qty{\norm{d_\mu^\top(P_\pi - I)} + \sqrt{2} \norm{D_\mu P_\pi}, \sqrt{2} \norm{d_\mu}}}, \sqrt{\frac{(1 - \xi) \eta}{\norm{D_\mu P_\pi}}}}, \\
C_1 &\doteq \max \qty{1, \frac{1}{\eta}} \frac{\norm{\bar h_1}}{\xi} + 1,
\end{align}
which completes the proof.
\end{proof}

\subsection{Proof of Theorem~\ref{thm:linear-q}}
\label{sec:proof-of-linear-q}
\begin{proof}
The proof is similar to the proof of Theorem~\ref{thm:op-td} in Section~\ref{sec:proof-of-td} but is more involving.
We define
\begin{align}
\label{eq:app-linear-q-reward}
h_\theta(s, a, s') &\doteq \big( r(s, a) + \gamma \max_{a'} x(s', a')^\top \bar{\theta} \big) x(s, a), \\
G_\theta(s, a, s') &\doteq x(s, a) x(s, a)^\top + \eta I,
\end{align}
where $\bar{\theta} \doteq \Gamma_{B_1}(\theta)$ is shorthand.
As $\theta_t \in B_1$ holds for all $t$,
we can rewrite the update of $w_t$ in Algorithm~\ref{alg:linear-q} as
\begin{align}
w_{t+1} = w_t + \alpha_t (h_{\theta_t}(Y_t) - G_{\theta_t}(Y_t)w_t).
\end{align}
The expected update given $\theta$ is then controlled by
% \notey{Remind the reader the policy $\mu_\theta$ is parametric and changing.}
% \notey{Any assumption on $D_{\mu_\theta}$? The policy has to cover all the states? Assumption ~\ref{assu:control-ergodic} ensures this?}
\begin{align}
\label{eq:app-linear-q-gain}
\bar{h}(\theta) &\doteq \E_{(s, a)\sim d_{\mu_\theta}, s' \sim p(\cdot|s, a)}[h_\theta(s, a, s')] \\
& = X^\top D_{\mu_\theta} r + \gamma X^\top D_{\mu_\theta} P_{\pi_{\bar{\theta}}} X \bar{\theta}, \\
\bar{G}(\theta) &\doteq \E_{(s, a)\sim d_{\mu_\theta}, s' \sim p(\cdot|s, a)}[G_\theta(s, a, s')] \\
& = X^\top D_{\mu_\theta} X + \eta I,
\end{align}
where Assumption~\ref{assu:control-ergodic} ensures the existence of $d_{\mu_\theta}$
and $\pi_\theta$ is the target policy, i.e.
a greedy policy with random tie breaking defined as follows.

Let $\mathcal{A}^{\max}_{s,\theta} \doteq \arg\max_a x(s, a)^\top \theta$ be the set of maximizing actions
for state $s$,
we define
\begin{align}
\pi_\theta(a|s) \doteq \begin{cases}
\frac{1}{|\mathcal{A}^{\max}_{s,\theta}|}, & a \in \mathcal{A}^{\max}_{s,\theta} \\
0, &\text{otherwise} 
\end{cases}.
\end{align}
Similar to the proof in Section~\ref{sec:proof-of-td},
we define
\begin{align}
\label{eq w star definition in linear q}
w^*(\theta) \doteq \bar{G}(\theta)^{-1}\bar{h}(\theta)
= (X^\top D_{\mu_\theta} X + \eta I)^{-1} X^\top D_{\mu_\theta} (r + \gamma P_{\pi_{\bar{\theta}}} X \bar{\theta})
\end{align}
and proceed to verify Assumptions~\ref{assu target net main net} -~\ref{assu target net contraction} to invoke Theorem~\ref{thm:target-net}.

Assumption~\ref{assu target net main net} is proved in Lemma~\ref{lem:main-linear-q}.

For Assumption~\ref{assu target net contraction},
Lemma~\ref{lem linear q lipschitz} shows that 
$w^*(\theta)$ is Lipschitz continuous in $\theta$ with 
\begin{align}
C_w \doteq \eta^{-1}\norm{X}\norm{r}L_D + \eta^{-2} \norm{X}^3 \norm{r} L_D + \gamma \eta^{-1} L_0 \norm{X}^2 + \gamma U_P \norm{X} R_{B_1} (\eta^{-1}\norm{X} L_D + \eta^{-2} \norm{X}^3 L_D)
\end{align}
being a Lipschitz constant.
Here $L_D$, $L_0$, and $U_P$ are positive constants detailed in the proof of Lemma~\ref{lem linear q lipschitz}.
Assuming 
\begin{align}
\label{eq:app-linear-q-x-cond1}
\norm{X} \leq 1 \qq{and} \gamma U_P \norm{X} R_{B_1} \leq 1,
\end{align}
we have
\begin{align}
C_w \leq \eta^{-2}\norm{X}(\eta \norm{r} L_D + \norm{r} L_D + \gamma \eta L_0 + \eta L_D +  L_D).
\end{align}
Take any $\xi \in (0, 1)$,
assuming
\begin{align}
\label{eq:app-linear-q-x-cond2}
\norm{X} \leq (1 - \xi) \eta^2 (\eta \norm{r} L_D + \norm{r} L_D + \gamma \eta L_0 + \eta  L_D + L_D)^{-1},
\end{align}
it then follows that $C_w \leq 1 - \xi$.
Assumptions~\ref{assu target net contraction}, therefore, holds.

We now select proper $R_{B_1}$ and $R_{B_2}$ to fulfill Assumption~\ref{assu target net boundedness}.
Similar to Lemma~\ref{lem linear q lipschitz} (see, e.g., the last three rows of Table~\ref{tab:linear-q-bounds} in the proof of Lemma~\ref{lem linear q lipschitz}),
we can easily get
\begin{align}
\norm{w^*(\theta)} \leq \eta^{-1} \norm{X}\norm{r} + \gamma \eta^{-1} \norm{X} U_P \norm{X}R_{B_1}.
\end{align}
Using~\eqref{eq:app-linear-q-x-cond1} yields
\begin{align}
\norm{w^*(\theta)} \leq \eta^{-1} \norm{X} (\norm{r} + 1).
\end{align}
% \notey{
% Here you have the option of keeping $\gamma$. 
% I think the $(1-\gamma)^{-1}$ should appear in the final bound. 
% We need to look into globally at the proof and see how this can be done. 
% }
For sufficiently large $R_{B_1}$, e.g.,
\begin{align}
\label{eq:app-linear-q-r-cond1}
R_{B_1} > \eta^{-1} \norm{X} (\norm{r} + 1) + \xi,
\end{align}
we have $\sup_\theta \norm{w^*(\theta)} < R_{B_1} - \xi$.
Taking $R_{B_2} \in (R_{B_1} - \xi, R_{B_1})$ then fulfills Assumption~\ref{assu target net boundedness}.

With Assumptions~\ref{assu target net lr} -~\ref{assu target net contraction} satisfied,
Theorem~\ref{thm:target-net} implies that there exists a unique $\theta_\infty$ such that
\begin{align}
w^*(\theta_\infty) = \theta_\infty \qq{and} \lim_{t\to \infty} \theta_t = \lim_{t \to \infty} w_t = \theta_\infty.
\end{align}

We now show what $\theta_\infty$ is.
We define
\begin{align}
&f(\theta) \doteq (X^\top D_{\mu_\theta} X + \eta I)^{-1} X^\top D_{\mu_\theta} (r + \gamma P_{\pi_{{\theta}}} X {\theta})
\end{align}
and consider a ball $B_0 \doteq \qty{\theta \in \R^K \mid \norm{\theta} \leq R_{B_0}}$ with $R_{B_0}$ to be tuned (for the Brouwer fixed-point theorem).
We have
\begin{align}
\norm{f(\theta)} &\leq \eta^{-1}\norm{X} \norm{r} + \gamma \eta^{-1} \norm{X}^2 U_P \norm{\theta}
% &= R_{B_0} - \left( R_{B_0} -  \gamma \eta^{-1} \norm{X}^2 U_P  -  \eta^{-1}\norm{X} \norm{r}\right).
\end{align}
% We want the second term be positive. 
% }
% \notey{check abs ftheta is not a contraction?}
Assuming
\begin{align}
\label{eq:app-linear-q-x-cond3}
\gamma \eta^{-1} \norm{X}^2 U_P < 1 - \xi,
\end{align}
we have
\begin{align}
\norm{f(\theta)} &\leq \eta^{-1}\norm{X} \norm{r} + (1 - \xi) \norm{\theta} \\
&= R_{B_0} - (R_{B_0} - (1 - \xi) \norm{\theta} - \eta^{-1}\norm{X} \norm{r} ).
\end{align}
Then for sufficiently large $R_{B_0}$, e.g.,
\begin{align}
R_{B_0} \geq \frac{\norm{X}\norm{r}}{\eta \xi},
\end{align}
we have 
\begin{align}
\theta \in B_0 \implies f(\theta) \in B_0.
\end{align}
% \notey{Note that $\norm{f(\theta)}$ is a contraction with a contraction factor $1-\xi$. } 
% \sz{This is true but it does not mean $f$ is a contraction.}
The Brouwer fixed-point theorem then asserts that there exists a $w^*_\eta \in B_0$ such that $f(w^*_\eta) = w^*_\eta$.
For sufficiently large $R_{B_1}$, e.g., 
\begin{align}
\label{eq:app-linear-q-r-cond2}
R_{B_1} > R_{B_0},
\end{align}
we have $\Gamma_{B_1}(w^*_\eta) = w^*_\eta$, i.e.,
$w^*_\eta$ is also a fixed point of $w^*(\cdot)$.
The contraction of $w^*(\cdot)$ then implies $\theta_\infty = w^*_\eta$. 
Rewriting $f(w^*_\eta) = w^*_\eta$ yields
\begin{align}
A_{\pi_{w^*_\eta}, \mu_{w^*_\eta}}w^*_\eta + \eta w^*_\eta - b_{\mu_{w^*_\eta}} &= \tb{0}.
\end{align}
In other words, $w^*_\eta$ is the unique solution of $(A_{\pi_{w}, \mu_w} + \eta I)w - b_{\mu_w} = \tb{0}$ inside $B_1$ (due to the contraction of $w^*(\cdot)$).
Combining~\eqref{eq:app-linear-q-x-cond1}~\eqref{eq:app-linear-q-x-cond2}~\eqref{eq:app-linear-q-x-cond3}~\eqref{eq:app-linear-q-r-cond1}~\eqref{eq:app-linear-q-r-cond2},
the desired constant is
\begin{align}
C_0 \doteq \min \{&1, \frac{1}{\gamma U_P R_{B_1}}, \frac{\eta (R_{B_1} - \xi)}{\norm{r} + 1}, \sqrt{\frac{\eta(1 - \xi)}{\gamma U_P}}, \frac{R_{B_1}\eta \xi}{\norm{r}}, \\
& \frac{(1 - \xi) \eta^2}{\eta \norm{r} L_D + \norm{r} L_D + \gamma \eta L_0 + \eta L_D + L_D}\},
\end{align}
which completes the proof.
As $R_{B_1}$ is usually large,
in general $C_0$ is $\mathcal{O}(R_{B_1}^{-1})$.
Though $C_0$ is potentially small, 
we can use small $\eta$ as well.
So a small $C_0$ (i.e., $\norm{X}$) does not necessarily implies a large regularization bias.
\end{proof}

\subsection{Proof of Theorem~\ref{thm:gradient-q}}
\label{sec:proof-of-gradient-q}
\begin{proof}
Let $\kappa_t \doteq [u_t^\top, w_t^\top]^\top$,
Algorithms~\ref{alg:gradient-q} implies that
\begin{align}
\kappa_{t+1} = \kappa_t + \alpha_t (h_{\theta_t}(Y_t) - G_{\theta_t}(Y_t) \kappa_t),
\end{align}
where
\begin{align}
G_\theta(s, a, s') &\doteq \mqty[x(s, a) x(s, a)^\top & -x(s, a)(\gamma \sum_{a'} \pi_\theta(a'|s')x(s, a) - x(s, a))^\top \\(\gamma \sum_{a'} \pi_\theta(a'|s')x(s, a) - x(s, a))x(s, a)^\top & \eta I], \\
h_\theta(s, a, s') &\doteq \mqty[x(s, a) r(s, a) \\ \tb{0}].
\end{align}
% \notey{typo here $x(s, a)^\top x(s, a)$ should be $x(s, a) x(s, a)^\top$}
We define
\begin{align}
C_{\mu_\theta} &\doteq X^\top D_{\mu_\theta} X, \\
\bar{G}(\theta) &\doteq \E_{(s, a)\sim d_{\mu_\theta}, s'\sim p(\cdot | s, a)}[G_\theta(s, a, s')] = \mqty[C_{\mu_{\theta}}  & A_{\pi_\theta, \mu_\theta} \\ -A_{\pi_\theta, \mu_\theta}^\top & \eta I], \\
\bar{h}(\theta) &\doteq \E_{(s, a)\sim d_{\mu_\theta}, s'\sim p(\cdot | s, a)}[h_\theta(s, a, s')] = \mqty[X^\top D_{\mu_\theta} r \\ \tb{0}], \\
\label{eq w star graident q}
w^*(\theta) &\doteq [\bar{G}(\theta)^{-1}\bar{h}(\theta)]_{K+1:2K} = (\eta I + A_{\pi_\theta, \mu_\theta}^\top C_{\mu_\theta}^{-1} A_{\pi_\theta, \mu_\theta})^{-1}A_{\pi_\theta, \mu_\theta}^\top C_{\mu_\theta}^{-1} X^\top D_{\mu_\theta}r,
\end{align}
% \notey{Have a comment here: The idea of our algorithm is to make the matrix $\bar{G}(\theta)$ S.P.D. In general $C$ can be any matrix that is S.P.D. It is called a preconditioner and helps improve the spectral property of matrix $G$. 
% BTW, the first gradient algorithms (in matrix-vector form) was in my preconditioning paper, 
% which minimizes $ || (A^\top A)w + A^\top  b ||$, which I called it "Grad TD". This paper was submitted a year and half earlier than GTD and never credited by these guys... It also inspired iLSTD. I think in which the preconditioner is the identity matrix. 
% In general I think it is possible to design a good behavior policy that can introduce a good preconditioner which can speed up the algorithm. 
% }
where $[\cdot]_{K+1:2K}$ is the subvector indexed from $K+1$ to $2K$.
% \notey{$C_{\mu_\theta}$ was not defined. }

We proceed to verifying Assumptions~\ref{assu target net main net},~\ref{assu target net boundedness}, \&~\ref{assu target net contraction} thus invoke Theorem~\ref{thm:target-net}.

Assumption~\ref{assu target net main net} is verified in Lemma~\ref{lem:main-gradient-q}. 
Lemma~\ref{lem w star gradient q lipschitz} shows that if $\norm{X} \leq 1$,
there exist a constant $L_w > 0$, 
which depends on $X$ through only $\frac{X}{\norm{X}}$,
%  but is independent of $\norm{X}$
% (i.e., scaling $X$ with any positive scalar does not affect $L_w$),
such that
\begin{align}
\norm{w^*(\theta_1) - w^*(\theta_2)} \leq L_w \norm{X} \norm{\theta_1 - \theta_2}.
\end{align}
As $w^*(\cdot)$ is independent of $R_{B_1}$,
so does $L_w$.
So as long as
\begin{align}
	\label{eq gq x condition}
\norm{X} \leq C_0 \doteq \min\qty{1, \frac{1 - \xi}{L_w}},
\end{align}
$w^*(\cdot)$ is contractive and Assumption~\ref{assu target net contraction} is satisfied.
Since $L_w$ depends on $X$ only through $\frac{X}{\norm{X}}$,
there are indeed $X$ satisfying \eqref{eq gq x condition}.
E.g.,
if some $X'$ does not satisfy \eqref{eq gq x condition},
we can simply scale $X'$ down by some scalar.
In the proof of Lemma~\ref{lem w star gradient q lipschitz},
we show $\sup_\theta \norm{C_{\mu_\theta}^{-1}} < \infty$.
It is then easy to see $\sup_\theta \norm{w^*(\theta)} < \infty$.
Consequently, we can choose sufficiently large $R_{B_1}$ and $R_{B_2}$ such that Assumption~\ref{assu target net boundedness} holds. 

With Assumptions~\ref{assu target net lr} -~\ref{assu target net contraction} satisfied,
Theorem~\ref{thm:target-net} implies that there exists a unique $w^*_\eta$ such that
\begin{align}
w^*(w^*_\eta) = w^*_\eta \qq{and} \lim_{t\to \infty} \theta_t = \lim_{t \to \infty} w_t = w^*_\eta.
\end{align}
Expanding $w^*(w^*_\eta) = w^*_\eta$ yields
\begin{align}
w^*_\eta = (A_{\pi_{w^*_\eta}, \mu_{w^*_\eta}}^\top C_{\mu_{w^*_\eta}}^{-1} A_{\pi_{w^*_\eta}, \mu_{w^*_\eta}} + \eta I)^{-1} A_{\pi_{w^*_\eta}, \mu_{w^*_\eta}}^\top C_{\mu_{w^*_\eta}}^{-1} b_{\mu_{w^*_\eta}},
\end{align}
which completes the proof.
% For a fixed $\theta$,
% the first order optimality condition on $w$ for $J(w, \theta)$ is
% \begin{align}
% \tb{0} &= \nabla_w J(w, \theta), \\
% \iff \tb{0} &= (A_{\pi_\theta, \mu_\theta} w - b_{\mu_\theta})^\top C_{\mu_\theta}^{-1} A_{\pi_\theta, \mu_\theta} + \eta w^\top, \\
% \iff \tb{0} &=(A_{\pi_\theta, \mu_\theta}^\top C_{\mu_\theta}^{-1} A_{\pi_\theta, \mu_\theta} + \eta I)w - A_{\pi_\theta, \mu_\theta}^\top C_{\mu_\theta}^{-1} b_{\mu_\theta},
% \end{align}
% from which it is easy to see $w^*_\eta$ is an isolated minimizer of $J(w, w^*_\eta)$.
\end{proof}

\subsection{Proof of Theorem~\ref{thm:diff-linear-q}}
\label{sec:proof-diff-linear-q}
\begin{proof}
The proof is combination of the proofs of Theorem~\ref{thm:diff-op-td} and Theorem~\ref{thm:linear-q}.
To avoid verbatim repetition,
in this proof,
we show only the existence of the constants $C_0$ and $C_1$ without showing the exact expressions.
% As $\theta_t \in B_1$ holds for all $t$, 
% $\norm{\theta^w_t} \leq R_{B_1}$ holds for all $t$ as well.
% Consider a ball $B_1' \doteq \qty{{x} \in \R^K \mid \norm{x} \leq R_{B_1}}$,
% we have $\Gamma_{B_1'}(\theta^w_t) = \theta^w_t$ for all $t$. 
We define 
\begin{align}
\theta &\doteq \mqty[\theta^r \\ \theta^w], u \doteq \mqty[\bar r \\ w], \mqty[\bar \theta^r \\ \bar \theta^w] \doteq \Gamma_{B_1}(\mqty[\theta^r \\ \theta^w]), \\
h_\theta(s, a, s') &\doteq \mqty[ r(s, a) \\ x(s, a) r(s, a)] + \mqty[0 & \sum_{a'} \pi_{\bar \theta^w}(a'|s')x(s', a')^\top \bar \theta^w - x(s, a)^\top \bar \theta^w\\
-x(s, a) \bar \theta^r & x(s, a) \sum_{a'} \pi_{\bar \theta^w}(a'|s')x(s', a')^\top \bar \theta^w], \\
G_\theta(s, a, s') &\doteq \mqty[1 & \tb{0}^\top \\ \tb{0} & x(s, a) x(s, a)^\top + \eta I],
\end{align}
We can then rewrite the update of $\bar r$ and $w$ in Algorithm~\ref{alg:diff-linear-q} as
\begin{align}
u_{t+1} = u_t + \alpha_t (h_{\theta_t}(Y_t) - G_{\theta_t}(Y_t) u_t).
\end{align}
In the rest of this proof,
we write $\mu_\theta$ and $\pi_\theta$ as shorthand for $\mu_{\theta^w}$ and $\pi_{\theta^w}$.
% \notey{Is this the same for the other proofs as well?}
We define
\begin{align}
\bar{h}(\theta) &\doteq \E_{(s, a) \sim d_{\mu_{\theta}}, s' \sim p(\cdot |s, a)}[h_\theta(s, a, s')] = \bar h_1(\theta) + \bar H_2(\theta), \\
\bar h_1(\theta) &\doteq \mqty[d_{\mu_\theta}^\top r \\ X^\top D_{\mu_\theta} r], \bar H_2(\theta) \doteq \mqty[0 & d_{\mu_\theta}^\top (P_{\pi_{\bar \theta^w}} - I) X \bar \theta^w \\ -(X^\top d_{\mu_\theta}) \bar \theta^r & X^\top D_{\mu_\theta} P_{\pi_{\bar \theta^w}} X \bar \theta^w], \\
\bar{G}(\theta) &\doteq \E_{(s, a) \sim d_{\mu_\theta}, s' \sim p(\cdot |s, a)}[G_\theta(s, a, s')] = \mqty[1 & \tb{0}^\top \\ \tb{0} & X^\top D_{\mu_\theta} X + \eta I], \\
\label{eq u star definition diff linear q}
u^*(\theta) &\doteq \bar{G}(\theta)^{-1} \bar h(\theta).
\end{align}

We proceed to verifying Assumptions~\ref{assu target net main net},~\ref{assu target net boundedness}, \&~\ref{assu target net contraction} to invoke Theorem~\ref{thm:target-net}.

Assumption~\ref{assu target net main net} is verified in Lemma~\ref{lem:main-diff-linear-q}.

For Assumption~\ref{assu target net contraction},
Lemma~\ref{lem diff linear q lipschitz} suggests that assuming $\norm{X} \leq 1$, $L_\mu \leq 1$, then
\begin{align}
C_u = \max\qty{1, \eta^{-1}} (\fO(\norm{X}) + \fO(L_\mu)) + \max\qty{1, \eta^{-2}} \fO(\norm{X})
\end{align}
is a Lipschitz constant of $u^*(\theta)$.
Take any $\xi \in (0,  1)$,
it is easy to see 
there exists positive constants $C_2$ and $C_3$ such that 
\begin{align}
\label{eq diff q inequality 1}
\norm{X} \leq C_2, L_\mu \leq C_3 \implies C_u \leq 1 - \xi.
\end{align}
Assumption~\ref{assu target net contraction}, therefore, holds.

We now select proper $R_{B_1}$ and $R_{B_2}$ to fulfill Assumption~\ref{assu target net boundedness}.
Using~\eqref{eq diff q inequality 1},
it is easy to see
\begin{align}
\norm{u^*(\theta)} \leq C_4 + (1 - \xi) R_{B_1}
\end{align}
for some positive constant $C_4$.
For sufficiently large $R_{B_1}$, e.g.,
\begin{align}
\label{eq diff q inequality 2}
R_{B_1} \geq \max \frac{C_4}{\xi} + 1,
\end{align}
we have $\sup_\theta \norm{u^*(\theta)} \leq R_{B_1} - \xi$.
Selecting $R_{B_2} \in (R_{B_1} - \xi, R_{B_1})$ then fulfills Assumption~\ref{assu target net boundedness}.

With Assumptions~\ref{assu target net lr} -~\ref{assu target net contraction} satisfied, Theorem~\ref{thm:target-net} then implies that there exists a unique $\theta_\infty$ such that
\begin{align}
u^*(\theta_\infty) = \theta_\infty \qq{and} \lim_{t\to \infty} \theta_t = \lim_{t \to \infty} u_t = \theta_\infty. 
\end{align}

We now show what $\theta_\infty$ is.
We define
\begin{align}
f(\theta) \doteq \bar G(\theta)^{-1}(\bar h_1(\theta) + \mqty[0 & d_{\mu_\theta}^\top (P_{\pi_{\theta}} - I) X  \\ -X^\top d_{\mu_\theta} & X^\top D_{\mu_\theta} P_{\pi_{\theta}} X ] \theta).
\end{align}
Similar to the proof of Theorem~\ref{thm:linear-q} in Section~\ref{sec:proof-of-linear-q},
we can use the Brouwer fixed point theorem to find a $u^*_\eta \in \Gamma_{B_1}$ such that $f(u^*_\eta) = u^*_\eta$ if 
\begin{align}
\label{eq diff q inquality 3}
R_{B_1} \geq C_5
\end{align}
for some constant $C_5$.
Then it is easy to see $u^*_\eta$ is also the fixed point of $u^*(\cdot)$, implying
$\theta_\infty = u^*_\eta$.
Rearranging terms of $u^*_\eta = f(u^*_\eta)$ yields 
\begin{align}
(\bar A_{\pi_w, \mu_w}+\eta I) w - \bar b_{\mu_w} = \tb{0},
\end{align}
% \notey{$\bar A, \bar b$ not defined. You may need to run over a check for all the proofs about $A, b, C$ and other commonly shared notations between proofs.  }
and the desired constants $C_0$ and $C_1$ can be deduced from~\eqref{eq diff q inequality 1},~\eqref{eq diff q inequality 2}, \&~\eqref{eq diff q inquality 3}. 
\end{proof}

\section{Convergence of Main Networks}
This section is a collection of several convergence proofs of main networks.
We first state a general convergence result regarding the convergence of time varying linear systems from \citet{konda2002thesis},
which will be repeatedly used.

Consider a stochastic process $\qty{Y_t}$ taking values in a finite space $\Y$.
Let $\qty{P_\theta \in \R^{|\Y| \times |\Y|} \mid \theta \in \R^K}$ be a parameterized family of transition kernels on $\Y$.
% \notey{Sounds like model-based RL? Can you give some description/example on $\theta$? like how it influence $P$, $h$ and $G$? You have so many forms of these quantities. Just take some for example. }
We update the parameter $w \in \R^K$ recursively as
\begin{align}
w_{t+1} \doteq w_t + \alpha_t (h_{\theta_t}(Y_t) - G_{\theta_t}(Y_t)w_t),
\end{align}
where $h_\theta: \Y \to \R^K$ and $G_\theta: \Y \to \R^{K \times K}$ are two vector- and matrix-valued functions.
\begin{assumption}
\label{assu:konda-markov}
$\Pr(Y_{t+1} \mid Y_0, w_0, \theta_0, \cdots, Y_t, w_t, \theta_t) = P_{\theta_t}(Y_t, Y_{t+1})$.
\end{assumption}
\begin{assumption}
\label{assu:konda-faster}
The learning rates $\qty{\alpha_t}$ is positive deterministic nonincreasing and satisfies
\begin{align}
\sum_t \alpha_t = \infty, \sum_t \alpha_t^2 < \infty.
\end{align}
\end{assumption}
\begin{assumption}
\label{assu:konda-slower}
The random sequence $\qty{\theta_t}$ satisfies 
\begin{align}
\norm{\theta_{t+1} - \theta_t} \leq \beta_t C,
\end{align}
where $C > 0$ is a constant, $\qty{\beta_t}$ is a deterministic sequence satisfying 
\begin{align}
\sum_t \Big(\frac{\beta_t}{\alpha_t}\Big)^d < \infty.
\end{align}
for some $d > 0$.
% \notey{So $\beta_t$ goes to zero faster than $\alpha_t$, 
% which means the change in $\theta$ is fairly small per time step. 
% }
\end{assumption}
\begin{assumption}
\label{assu:konda-poisson}
For each $\theta$, 
there exists $\bar{h}(\theta) \in \R^K, \bar{G}(\theta) \in \R^{K \times K}, \hat{h}_\theta: \Y \to \R^K, \hat{G}_\theta: \Y \to \R^{K \times K}$ such that for each $y \in \Y$,
\begin{align}
\hat{h}_\theta(y) &= h_\theta(y) - \bar{h}(\theta) + \sum_{y'} P_\theta(y, y') \hat{h}_\theta(y'),\\
\hat{G}_\theta(y) &= G_\theta(y) - \bar{G}(\theta) + \sum_{y'} P_\theta(y, y') \hat{G}_\theta(y').
\end{align}
\end{assumption}
% \notey{The above assumption doesn't interpret well. What does it mean? $\bar{h}$ is the true value? It's best if you can explain it, otherwise understanding the next assumptions is hard. }
\begin{assumption}
\label{assu:konda-bounded1}
$\sup_\theta \norm{\bar{h}(\theta)} < \infty, \sup_\theta \norm{\bar{G}(\theta)} < \infty$.
\end{assumption}
\begin{assumption}
\label{assu:konda-bounded2}
For any $f_\theta \in \qty{h_\theta, \hat{h}_\theta, G_\theta, \hat{G}_\theta}$,
$\sup_{\theta, y} \norm{f_{\theta}(y)} < \infty$.
\end{assumption}
\begin{assumption}
\label{assu:konda-lipschitz1}
There exists some constant $C > 0$ such that
\begin{align}
\norm{\bar{h}(\theta) - \bar{h}(\theta')} \leq C \norm{\theta - \theta'}, \norm{\bar{G}(\theta) - \bar{G}(\theta')} \leq C \norm{\theta - \theta'}.
\end{align}
\end{assumption}
\begin{assumption}
\label{assu:konda-lipschitz2}
There exists some constant $C > 0$ such that
for any $f_\theta \in \qty{h_\theta, \hat{h}_\theta, G_\theta, \hat{G}_\theta}$ and $y \in \Y$,
\begin{align}
\norm{f_\theta(y) - f_{\theta'}(y)} \leq C \norm{\theta - \theta'}.
\end{align}
\end{assumption}
\begin{assumption}
\label{assu:konda-pd}
There exists a constant $\zeta > 0$ such that for all $w, \theta$, $$w^\top \bar{G}(\theta)w \geq \zeta \norm{w}^2.$$
\end{assumption}
% \notey{This requires the matrix $G$ to be positive definite. 
% You don't need it to be positive definite, right?
% This assumption is assuming the smallest eigenvalues of $G$ is bigger than $\zeta$. 
% It influences the convergence rate of your algorithms. 
% If $\zeta$ is small, the rate is slow. 
% It is also related to the condition number of $G$. 
% I think somewhere you have some condition number $\kappa$? 
% It may be interesting to relate to here and see if you have a general result. 
% }
\begin{theorem}
\label{thm:konda}
(Theorem 3.2 in \citet{konda2002thesis})
Under Assumptions~\ref{assu:konda-markov}-~\ref{assu:konda-pd}, almost surely,
\begin{align}
\lim_{t \to \infty} \norm{\bar{h}(\theta_t) - \bar{G}(\theta_t)w_t} = 0.
\end{align}
\end{theorem}
We start with the convergence of the main network in $Q$-evaluation with a Target Network (Algorithm~\ref{alg:off-policy-td}).
% \notey{Does it mean you will go over one alg. after another? Looks like so. You need to give some direction here. It's long here :)}
\begin{lemma}
\label{lem:main-of-td}
Almost surely, $\lim_{t\to \infty} \norm{w_t - w^*(\theta_t)} = 0$.
\end{lemma}
\begin{proof}
This proof is a special case of the proof of Lemma~\ref{lem:main-linear-q}.
% \notey{Why not presenting Lemma~\ref{lem:main-linear-q} first instead? It's not good to refer to later. }
We, therefore, omit it to avoid verbatim repetition.
\end{proof}
We then show the convergence of the main network in Differential $Q$-evaluation with a Target Network (Algorithm~\ref{alg:differential-off-policy-td}).
\begin{lemma}
\label{lem:main-diff-of-td}
Almost surely, $\lim_{t\to \infty} \norm{u_t - u^*(\theta_t)} = 0$.
\end{lemma}
\begin{proof}
This proof is the same as the proof of Lemma~\ref{lem:main-of-td} up to change of notations.
We, therefore, omit it to avoid verbatim repetition.
\end{proof}

We now proceed to show the convergence of the main network in $Q$-learning with a Target Network (Algorithm~\ref{alg:linear-q}).
In Lemma~\ref{lem:main-linear-q} and its proof, 
we continue using the notations in Theorem~\ref{thm:linear-q} and its proof (Section~\ref{sec:proof-of-linear-q}).
\begin{lemma}
\label{lem:main-linear-q}
Almost surely, $\lim_{t\to \infty} \norm{w_t - w^*(\theta_t)} = 0$.
\end{lemma}
\begin{proof}
We proceed by verifying Assumptions~\ref{assu:konda-markov}-~\ref{assu:konda-pd}.
Let $Y_t \doteq (S_t, A_t, S_{t+1})$ and 
\begin{align}
\Y \doteq \qty{(s, a, s') \mid s \in \mathcal{S}, a \in \mathcal{A}, s' \in \mathcal{S}, p(s'|s, a) > 0}.
\end{align}
Let $y_1 = (s_1, a_1, s_1'), y_2 = (s_2, a_2, s_2')$,
we define $P_\theta \in \R^{|\Y| \times |\Y|}$ as 
\begin{align}
\label{eq:app-linear-q-transition-kernel}
P_\theta(y_1, y_2) \doteq \mathbb{I}_{s_1' = s_2} \mu_\theta(a_2|s_2) p(s_2'|s_2, a_2).
\end{align}
% \notey{$s_1, a_1$ are not relevant in this definition. 
% I saw in the literature people have used transition from state-action pair to state-action pair. 
% I haven't seen this, but it seems valid. 
% I think if you sum over $s_2$, then you get that state-action transition model. 
% }
Recall in Section~\ref{sec:proof-of-linear-q}, we define
\begin{align}
\bar{\theta} &\doteq \Gamma_{B_1}(\theta), \\
h_\theta(s, a, s') &\doteq \big( r(s, a) + \gamma \max_{a'} x(s', a')^\top \bar{\theta} \big) x(s, a), \\
G_\theta(s, a, s') &\doteq x(s, a) x(s, a)^\top + \eta I, \\
\bar{h}(\theta) &\doteq \E_{(s, a)\sim d_{\mu_\theta}, s' \sim p(\cdot|s, a)}[h_\theta(s, a, s')] \\
& = X^\top D_{\mu_\theta} r + \gamma X^\top D_{\mu_\theta} P_{\pi_{\bar{\theta}}} X \bar{\theta}, \\
\bar{G}(\theta) &\doteq \E_{(s, a)\sim d_{\mu_\theta}, s' \sim p(\cdot|s, a)}[G_\theta(s, a, s')] \\
& = X^\top D_{\mu_\theta} X + \eta I.
\end{align}
% \notey{Up to this point, $x(s, a)$ is not defined. What assumptions on the features?
% Looks like there is some relationship between $\eta$ and $\zeta$. }
% \notey{Note $P_{\pi_{\bar{\theta}}}$. Previously a few lines before you had $P_\theta$. 
% Also do you mean $\pi_{\bar{\theta}}$ is the greedy policy?
% }
According to Algorithm~\ref{alg:linear-q}, 
% \notey{Here needs more explanation about the Markov below. }
we have
\begin{align}
\Pr(Y_{t+1} | Y_0, w_0, \theta_0, \dots, Y_t, w_t, \theta_t) &= \Pr(S_{t+1}, A_{t+1}, S_{t+2} | S_t, A_t, S_{t+1}, \theta_t) \\
&= \mu_{\theta_t}(A_{t+1} | S_{t+1}) p(S_{t+2} | S_{t+1}, A_{t+1}) \\
&= P_{\theta_t}(Y_t, Y_{t+1}).
\end{align}
Assumption~\ref{assu:konda-markov} therefore holds.
Assumptions \ref{assu:konda-faster} \& \ref{assu:konda-slower} are satisfied automatically by Assumptions~\ref{assu target net lr},~\ref{assu target net all lrs}, and Lemma~\ref{lem target net changes slowly}.

Consider a Markov Reward Process (MRP) in $\Y$ with transition kernel $P_\theta$ and reward function $h_\theta(y)_i$, the $i$-th element of $h_\theta(y)$ (we need an MRP for each $i$).
% \notey{do you have a MRP for each $i$?}
The stationary distribution of this MRP is $d_{\Y, \theta}(s, a, s') = d_{\mu_\theta}(s, a) p(s'|s, a)$.
The reward rate 
% \notey{What is reward rate?}
of this MRP is $\sum_y d_{\Y, \theta}(y) h_\theta(y)_i = \bar{h}(\theta)_i$,
the $i$-the element of $\bar{h}(\theta)$.
We define 
\begin{align}
\hat{h}_\theta(y) \doteq \E[\sum_{k=0}^\infty h_\theta(Y_k) - \bar{h}(\theta) | Y_0 = y, P_\theta].
\end{align}
By definition, $\hat{h}_\theta(y)_i$ is the differential value function of this MRP. 
% \notey{What is differential value function? reference?}
% \notey{Use $\E\left[ \right]$. }
% \notey{$\bar{h}$ should have argument $y$. }
% \notey{Are you sure there is expectation in defining $\hat{h}_\theta(y) $?
% After expectation, is it not zero?
% }
The differential Bellman equation~\eqref{eq differential bellman} then implies the first equation in Assumption~\ref{assu:konda-poisson} holds.
Similarly, we define
\begin{align}
\hat{G}_\theta(y) \doteq \E[\sum_{k=0}^\infty G_\theta(Y_k) - \bar{G}(\theta) | Y_0 = y, P_\theta],
\end{align}
the second equation in Assumption~\ref{assu:konda-poisson} holds as well.

Assumption~\ref{assu:konda-bounded1} follows directly from the definition of $\bar{h}(\theta)$ and $\bar{G}(\theta)$.
The boundedness of $h_\theta(y)$ and $G_\theta(y)$ in Assumption~\ref{assu:konda-bounded2} follows directly from their definitions. 
The differential value function of the MRP can be equivalently expressed as (see, e.g, Eq (8.2.2) in \citet{puterman2014markov})
% \notey{This should be mentioned earlier. 
% Many people don't know this concept. Like me. 
% }
\begin{align}
\label{eq:app-of-td-gain-bias}
\hat{h}_{\theta,i} &= (I - P_\theta + \tb{1}d_{\Y, \theta}^\top)^{-1}(I - \tb{1}d_{\Y, \theta}^\top) h_{\theta, i}, \\
\hat{G}_{\theta,ij} &= (I - P_\theta + \tb{1}d_{\Y, \theta}^\top)^{-1}(I - \tb{1}d_{\Y, \theta}^\top) G_{\theta, ij},
\end{align}
% \notey{This indicates there should be expectation in $\hat{h}$. }
where $h_{\theta, i}$ denotes the vector in $\R^{|\Y|}$ whose $y$-th element is $h_{\theta}(y)_i$ and $\hat{h}_{\theta, i}, \hat{G}_{\theta, ij}, G_{\theta, ij}$ are similarly defined as $h_{\theta, i}$.
Note the existence of the matrix inversion results directly from the ergodicity of the chain (see, e.g., Section A.5 in the appendix of \citet{puterman2014markov}).
To show the boundedness of $\hat h_\theta(y)$ and $\hat G_\theta(y)$ required in Assumption~\ref{assu:konda-bounded2},
it suffices to show the boundedness of $(I - P_\theta + \tb{1}d_{\Y, \theta}^\top)^{-1}(I - \tb{1}d_{\Y, \theta}^\top)$.
The intuition is simple:
it is a continuous function in a compact set $\fP$ (defined in Assumption \ref{assu:control-ergodic}), 
so it obtains its maximum.
% \notey{Why the matrix is invertible?}
We now formalize this intuition.
We first define a function $f: \fP \to \R^{|\Y| \times |\Y|}$ that maps a transition matrix in $\mathcal{S} \times \mathcal{A}$ to a transition matrix in $\Y$.
For $P \in \fP, y_1 \in \Y, y_2 \in \Y$, 
we have $f(P)(y_1, y_2) \doteq \mathbb{I}_{s_1' = s_2} P((s_1, a_1), (s_2, a_2)) p(s_2' | s_2, a_2) / p(s_1'|s_1, a_1)$.
% \notey{This mapping is worth more explanation. }
Let $\fP^\Y \doteq \qty{f(P) \mid P \in \fP}$.
As $f$ is continuous and $\fP$ is compact, 
it is easy to see $\fP^\Y$ is also compact.
It is easy to verify that $f(P)$ is a valid transition kernel in $\Y$.
As Assumption~\ref{assu:control-ergodic} asserts $P$ is ergodic,
we use $d_{P}$ to denote the stationary distribution of the chain induced by $P$.
Let $d_{f(P)}(s, a, s') \doteq d_{P}(s, a)p(s'|s, a)$,
it is easy to verify that $d_{f(P)}$ is the stationary distribution of the chain induced by $f(P)$,
from which it is easy to see the chain in $\Y$ induced by $f(P)$ is also ergodic, i.e.,
any $P' \in \fP^\Y$ is ergodic.
Consider the function $g(P') \doteq (I - P' + \textbf{1} d_{P'}^\top)^{-1}(I - \textbf{1} d_{P'}^\top)$.
The ergodicity of $P'$ ensures $g(P')$ is well defined.
As $\fP^\Y$ is compact and $g$ is continuous,
$g$ attains its maximum in $\fP^\Y$,
say, e.g., $U_g$.
For any $\theta \in \R^K$, 
it is easy to verify that $P_\theta = f(P_{\mu_\theta}) \in \fP^\Y$,
where $P_\theta$ is defined in~\eqref{eq:app-linear-q-transition-kernel}.
Consequently, $\norm{g(P_\theta)} \leq U_g$.
The boundedness of $\hat{h}_\theta(y)$ and $\hat{G}_\theta(y)$ then follows directly from the boundedness of $h_\theta(y)$ and $G_\theta(y)$.
Assumption~\ref{assu:konda-bounded2} therefore holds.
% \begin{align}
% |\hat{h}_\theta(y)_i| &= |\sum_{k=0}^\infty \sum_{z \in \Y} \Big(\Pr(Y_k = z | Y_0 = y, P_\theta) - d_{\Y, \theta}(z) \Big)h_\theta(z)_i| \\
% &\leq \sum_{k=0}^\infty |\sum_{z \in \Y} \Big(\Pr(Y_k = z | Y_0 = y, P_\theta) - d_{\Y, \theta}(z) \Big)h_\theta(z)_i| \\
% &\leq \sum_{k=0}^\infty  |\sum_{z \in \Y} \Big(\Pr(Y_k = z | Y_0 = y, P_\theta) - d_{\Y, \theta}(z) \Big)| \norm{h_{\theta,i}}_\infty \qq{(Holder's inequality)} \\
% &\leq \sum_{k=0}^\infty C_\mathcal{P} \varrho^k \norm{h_{\theta,i}}_\infty \qq{(Assumption~\ref{assu:control-ergodic})} \\
% &= \frac{C_\mathcal{P}}{1 - \varrho} \norm{h_{\theta,i}}_\infty.
% \end{align}
% The boundedness of $\hat{h}_\theta(y)$ then follows immediately from the boundedness of $h_\theta(y)$.
% Similarly, we can verify $\hat{G}_\theta(y)$ is also bounded. 
% Assumption~\ref{assu:konda-bounded2} is then satisfied.

The Lipschitz continuity of $h_\theta(y), \bar{h}(\theta), G_\theta(y), \bar{G}(\theta)$ in $\theta$ follows directly from the Lipschitz continuity of $D_{\mu_\theta}$ (Lemma~\ref{lem ergodic dist lipschitz}), 
the Lipschitz continuity of $P_{\pi_{\bar{\theta}}}X \bar{\theta}$ (the proof of Lemma~\ref{lem linear q lipschitz}), and Lemma~\ref{lem:product-lipschitz}.
We can easily show the boundedness of $(I - P_\theta + \tb{1}d_{\Y, \theta}^\top)^{-1}$ similar to the boundedness of $(I - P_\theta + \tb{1}d_{\Y, \theta}^\top)^{-1}(I - \tb{1}d_{\Y, \theta}^\top)$.
% \notey{There should be some proof for the invertibility. }
The Lipschitz continuity of $\hat{h}_\theta(y)$ and $\hat{G}_\theta(y)$ in $\theta$ is then an exercise of Lemmas~\ref{lem:product-lipschitz} \&~\ref{lem:inverse-lipschitz}.
Assumptions~\ref{assu:konda-lipschitz1} \&~\ref{assu:konda-lipschitz2} are then satisfied.

Assumptions~\ref{assu:konda-pd} follows directly from the positive definiteness of $\bar{G}(\theta)$.

With Assumptions~\ref{assu:konda-markov}-~\ref{assu:konda-pd} satisfied, 
invoking Theorem~\ref{thm:konda} completes the proof. 
\end{proof}

We now show the convergence of the main network in Gradient $Q$-learning with a Target Network (Algorithm~\ref{alg:gradient-q}).
In Lemma~\ref{lem:main-gradient-q} and its proof, 
we continue using the notations in Theorem~\ref{thm:gradient-q} and its proof (Section~\ref{sec:proof-of-gradient-q}).
\begin{lemma}
\label{lem:main-gradient-q}
Almost surely, $\lim_{t\to \infty} \norm{w_t - w^*(\theta_t)} = 0$.
\end{lemma}

\begin{proof}
We proceed by verifying Assumptions~\ref{assu:konda-markov}-\ref{assu:konda-pd}.
Assumptions~\ref{assu:konda-markov}-\ref{assu:konda-lipschitz2} can be verified in the same way as the proof of Lemma~\ref{lem:main-linear-q}.
We, therefore, omit it to avoid verbatim repetition. 
For Assumption~\ref{assu:konda-pd},
consider $\kappa \doteq [u^\top, w^\top]^\top$,
we have
\begin{align}
\kappa^\top \bar{G}(\theta) \kappa = u^\top C_{\mu_\theta} u + \eta \norm{w}^2,
\end{align}
where $C_{\mu_\theta} = X^\top D_{\mu_\theta} X$.
We first show that 
there exists a constant $C_1 > 0$ such that
$u^\top C_{\mu_\theta} u \geq C_1 \norm{u}^2$ holds for any $u$ and $\theta$,
which is equivalent to showing for some $C_1 > 0$,
\begin{align}
u^\top C_{\mu_\theta} u \geq C_1 
\end{align}
holds for any $u \in \mathcal{U} \doteq \qty{u\mid\norm{u} = 1}$ and $\theta \in \R^K$.
% \notey{So $\zeta, \eta, C_1$ relate here. 
% The key condition is on $C_1$. 
% }
Consider $\fP$ defined in \ref{assu:control-ergodic}.
For any $P \in \fP$, 
we use $D_P$ to denote a diagonal matrix whose diagonal term is the stationary distribution of the chain induced by $P$.
Assumptions~\ref{assu full rank X} \&~\ref{assu:control-ergodic} then ensure that $X^\top D_P X$ is positive definite,
i.e.,
\begin{align}
u^\top X^\top D_P X u > 0
\end{align}
holds for any $u \in \mathcal{U}$ and $P \in \fP$.
As $u^\top X^\top D_P X u$ is a continuous function in both $u$ and $P$,
it obtains its minimum in the compact set $\mathcal{U} \times \fP$,
say, e.g., $C_1 > 0$,
i.e.
\begin{align}
u^\top X^\top D_P X u \geq C_1
\end{align}
holds for any $u \in \mathcal{U}$ and $P \in \fP$.
% \notey{I think this $C_1$ is different from the $C_1$ above. }
It is then easy to see
\begin{align}
u^\top C_{\mu_\theta} u \geq C_1 \norm{u}^2
\end{align}
holds for any $u \in \R^K$ and $\theta \in \R^K$,
implying
\begin{align}
\kappa^\top \bar{G}(\theta) \kappa \geq \min\qty{\eta, C_1} \norm{\kappa}^2.
\end{align}
Assumption~\ref{assu:konda-pd}, therefore, holds.
Theorem~\ref{thm:konda} then implies that
\begin{align}
\lim_{t \to \infty} \norm{\kappa_t - \bar G(\theta_t)^{-1} \bar h(\theta_t) } = 0,
\end{align}
which completes the proof.
\end{proof}

We now show the convergence of the main network in Differential $Q$-learning with a Target Network (Algorithm~\ref{alg:diff-linear-q}).
\begin{lemma}
\label{lem:main-diff-linear-q}
Almost surely, $\lim_{t\to \infty} \norm{u_t - u^*(\theta_t)} = 0$.
\end{lemma}
\begin{proof}
The proof is the same as the proof of Lemma~\ref{lem:main-linear-q} up to change of notations.
We, therefore, omit it to avoid verbatim repetition.
\end{proof}

\section{Technical Lemmas}
\label{sec:technical_lemmas}
\begin{lemma}
\label{lem:product-lipschitz}
Let $f_1(x)$ and $f_2(x)$ be two Lipschitz continuous functions with Lipschitz constants $C_1$ and $C_2$. 
If they are also bounded by $U_1$ and $U_2$, then
their product $f_1(x)f_2(x)$ is also Lipschitz continuous with $C_1U_2 + C_2U_1$ being a Lipschitz constant.
\end{lemma}
\begin{proof}
\begin{align}
\norm{f_1(x)f_2(x) - f_1(y)f_2(y)} &\leq \norm{f_1(x)(f_2(x) - f_2(y))} + \norm{f_2(y)(f_1(x) - f_1(y))} \\
&\leq (U_1C_2 + U_2C_1) \norm{x - y}.
\end{align}
\end{proof}

\begin{lemma}
\label{lem:inverse-lipschitz}
$\norm{Y_1^{-1} - Y_2^{-1}} \leq \norm{Y_1^{-1}} \norm{Y_1 - Y_2} \norm{Y_2^{-1}}$.
\end{lemma}
\begin{proof}
\begin{align}
\norm{Y_1^{-1} - Y_2^{-1}} = \norm{Y_1^{-1}(Y_1 - Y_2)Y_2^{-1}} \leq \norm{Y_1^{-1}} \norm{Y_1 - Y_2} \norm{Y_2^{-1}}.
\end{align}
\end{proof}

\begin{lemma}
\label{lem ergodic dist lipschitz}
Under Assumption~\ref{assu:control-ergodic} \&~\ref{assu:control-parameterization},
$D_{\mu_\theta}$ is Lipschitz continuous in $\theta$.
\end{lemma}
\begin{proof}
By the definition of stationary distribution,
we have
$\mqty[P_{\mu_\theta}^\top - I \\ \tb{1}^\top] d_{\mu_\theta} = \mqty[\tb{0} \\ 1]$.
Let $L(P) \doteq \mqty[P^\top - I \\ \tb{1}^\top]$.
Assumption~\ref{assu:control-ergodic} ensures that for any $\theta$,
$L(P_{\mu_\theta})$ has full column rank.
So
\begin{align}
d_{\mu_\theta} = \big(L(P_{\mu_\theta})^\top L(P_{\mu_\theta})\big)^{-1} L(P_{\mu_\theta})^\top \mqty[\tb{0} \\ 1].
\end{align}
It is easy to see $L(P_{\mu_\theta})$ is Lipschitz continuous in $\theta$ and $\sup_\theta \norm{L(P_{\mu_\theta})} < \infty$.
We now show $\big(L(P_{\mu_\theta})^\top L(P_{\mu_\theta})\big)^{-1}$ is Lipschitz continuous in $\theta$ and is uniformly bounded for all $\theta$,
invoking Lemma~\ref{lem:product-lipschitz} with which will then complete the proof.
Note
\begin{align}
\big(L(P_{\mu_\theta})^\top L(P_{\mu_\theta})\big)^{-1} = \frac{adj\big(L(P_{\mu_\theta})^\top L(P_{\mu_\theta})\big)}{\det\big(L(P_{\mu_\theta})^\top L(P_{\mu_\theta})\big)},
\end{align} 
where $adj(\cdot)$ denotes the adjugate matrix.
By the properties of determinants and adjugate matrices,
it is easy to see both the numerator and the denominator are Lipschitz continuous in $\theta$ and are uniformly bounded for all $\theta$.
It thus remains to show the denominator is bounded away from $0$.
Consider the compact set $\fP$ defined in Assumption~\ref{assu:control-ergodic},
for any $P \in \fP$,
Assumption~\ref{assu:control-ergodic} ensures $L(P)^\top L(P)$ is invertible,
i.e., $|\det(L(P)^\top L(P))| > 0$.
As $|\det(L(P)^\top L(P))|$ is continuous in $P$,
so it obtains its minimum in the compact set $\fP$,
say $C_0 > 0$.
It then follows $\inf_\theta |\det(L(P_{\mu_\theta})^\top L(P_{\mu_\theta}))| \geq C_0 > 0$,
which completes the proof.
\end{proof}

\begin{lemma}
\label{lem linear q lipschitz}
The $w^*(\theta)$ defined in~\eqref{eq w star definition in linear q} is Lipschitz continuous in $\theta$.
\end{lemma}
\begin{proof}
Recall
\begin{align}
w^*(\theta) \doteq \bar{G}(\theta)^{-1}\bar{h}(\theta)
= (X^\top D_{\mu_\theta} X + \eta I)^{-1} X^\top D_{\mu_\theta} (r + \gamma P_{\pi_{\bar{\theta}}} X \bar{\theta}).
\end{align}
We first show $P_{\pi_\theta}X\theta$ is Lipschitz continuous in $\theta$.
By definition of $\pi_\theta$,
\begin{align}
(P_{\pi_{{\theta}}} X {\theta})(s, a) &= \sum_{s'}p(s'|s, a) \max_{a'} x(s', a')^\top \theta.
% &= \sum_{s'}p(s'|s, a) \norm{X_{s'} \theta}_\infty,
\end{align}
Let $X_{s'} \in \R^{|\mathcal{A}| \times K}$ be a matrix whose $a'$-th row is $x(s', a')^\top$.
Then
\begin{align}
|(P_{\pi_{{\theta_1}}} X {\theta_1})(s, a) - (P_{\pi_{{\theta_2}}} X {\theta_2})(s, a)|
=& |\sum_{s'}p(s'|s, a) (\max_{a'}x(s', a')^\top \theta_1 - \max_{a'}x(s', a')^\top \theta_2 )| \\
\leq& \sum_{s'}p(s'|s, a) |\max_{a'}x(s', a')^\top \theta_1 - \max_{a'}x(s', a')^\top \theta_2 | \\
\leq& \sum_{s'}p(s'|s, a) \max_{a'}|x(s', a')^\top \theta_1 - x(s', a')^\top \theta_2 | \\
% \leq&\sum_{s'}p(s'|s, a) |\norm{X_{s'} \theta_1}_\infty - \norm{X_{s'} \theta_2}_\infty| \\
\leq& \sum_{s'}p(s'|s, a) \norm{X_{s'} \theta_1 - X_{s'} \theta_2}_\infty \\
\leq& \sum_{s'}p(s'|s, a) \norm{X_{s'}}_\infty \norm{\theta_1 - \theta_2}_\infty \\
\leq& \norm{X}_\infty \norm{\theta_1 - \theta_2}_\infty, \\
\implies \norm{P_{\pi_{{\theta_1}}} X {\theta_1} - P_{\pi_{{\theta_2}}} X {\theta_2}}_\infty \leq& \norm{X}_\infty \norm{\theta_1 - \theta_2}_\infty.
\end{align}
The equivalence between norms then asserts that there exists a constant $L_0 > 0$ such that 
\begin{align}
\norm{P_{\pi_{{\theta_1}}} X {\theta_1} - P_{\pi_{{\theta_2}}} X {\theta_2}} \leq& L_0 \norm{X} \norm{\theta_1 - \theta_2}.
\end{align}
It is easy to see $L_0 \norm{X}$ is also a Lipschitz constant of $P_{\pi_{\bar \theta}} X \bar \theta$ by the property of projection.

Lemma~\ref{lem ergodic dist lipschitz} ensures that $D_{\mu_\theta}$ is Lipschitz continuous in $\theta$ and we use $L_D$ to denote a Lipschitz constant.
We remark that if we assume $\norm{X} \leq 1$,
we can indeed select an $L_D$ that is independent of $X$.
To see this,
let $L_\mu$ be the Lipschitz constant in Assumption~\ref{assu:control-parameterization},
then we have
\begin{align}
	\abs{\mu_\theta(a|s) - \mu_{\theta'}(a|s)} \leq L_\mu \norm{X_s \theta - X_s \theta'}_\infty \leq L_\mu \norm{X_s}_\infty \norm{\theta - \theta'}_\infty \leq L_\mu \norm{X}_\infty \norm{\theta - \theta'}_\infty.
\end{align}
Due to the equivalence between norms,
there exists a constant $L_\mu' > 0$ such that
\begin{align}
	\label{eq lipschitz constant independent of x}
	\abs{\mu_\theta(a|s) - \mu_{\theta'}(a|s)} \leq L_\mu' \norm{X} \norm{\theta - \theta'} \leq L_\mu' \norm{\theta - \theta'}.
\end{align}

We can now use Lemmas~\ref{lem:product-lipschitz} \&~\ref{lem:inverse-lipschitz} to compute the bounds and Lipschitz constants for several terms of interest,
which are detailed in Table~\ref{tab:linear-q-bounds}.
From Table~\ref{tab:linear-q-bounds} and Lemma~\ref{lem:product-lipschitz},
a Lipschitz constant of $w^*(\theta)$ is 
\begin{align}
C_w \doteq \big( \eta^{-1}\norm{X}\norm{r}L_D + \eta^{-2} \norm{X}^3 \norm{r} L_D \big) + \eta^{-1} \norm{X} \gamma L_0 \norm{X} + \gamma U_P \norm{X} R_{B_1} (\eta^{-1}\norm{X} L_D + \eta^{-2} \norm{X}^3 L_D),
\end{align}
which completes the proof.

\begin{table}[h]
\centering
\begin{tabular}{ c |c | c}
  & Bound & Lipschitz constant  \\ \hline 
$D_{\mu_\theta}$ & $1$ & $L_D$ \\ \hline
$(X^\top D_{\mu_\theta}X + \eta I)^{-1}$ & $\eta^{-1}$ & $\eta^{-2} \norm{X}^2 L_D$ \\ \hline
$X^\top D_{\mu_\theta}r$ & $\norm{X}\norm{r}$ & $\norm{X}\norm{r}L_D$ \\ \hline
$(X^\top D_{\mu_\theta}X + \eta I)^{-1}X^\top D_{\mu_\theta}r$ & $\eta^{-1}\norm{X}\norm{r}$ & $\eta^{-1}\norm{X}\norm{r}L_D + \eta^{-2} \norm{X}^3 \norm{r} L_D$ \\ \hline
$(X^\top D_{\mu_\theta}X + \eta I)^{-1}X^\top D_{\mu_\theta}$ & $\eta^{-1}\norm{X}$ & $\eta^{-1}\norm{X} L_D + \eta^{-2} \norm{X}^3 L_D$ \\ \hline
$\gamma P_{\pi_{\bar{\theta}}} X^\top \bar{\theta}$ & $\gamma U_P \norm{X} R_{B_1}$ & $\gamma L_0 \norm{X}$ \\ \hline
\end{tabular}
\caption{\label{tab:linear-q-bounds}
$U_P \doteq \sup_\theta \norm{P_{\pi_{\theta}}}$.
} 
\end{table}
\end{proof}

\begin{lemma}
\label{lem w star gradient q lipschitz}
Under Assumptions~\ref{assu full rank X},~\ref{assu:control-ergodic} -~\ref{assu target policy continuous}, 
if $\norm{X} \leq 1$, 
then $w^*(\theta)$ defined in~\eqref{eq w star graident q} satisfies
\begin{align}
\norm{w^*(\theta_1) - w^*(\theta_2)} \leq \norm{X}L_w \norm{\theta_1 - \theta_2},
\end{align}
where $L_w$ is a positive constant that depends on $X$ through only $\frac{X}{\norm{X}}$.
\end{lemma}
\begin{proof}
We first recall that if $\norm{X} \leq 1$, 
the Lipschitz constants for both $\mu_\theta$ and $\pi_\theta$ in $\theta$ can be selected to be independent of $X$ (c.f. \eqref{eq lipschitz constant independent of x}).
Recall
\begin{align}
w^*(\theta) = (\eta I + A_{\pi_\theta, \mu_\theta}^\top C_{\mu_\theta}^{-1} A_{\pi_\theta, \mu_\theta})^{-1}A_{\pi_\theta, \mu_\theta}^\top C_{\mu_\theta}^{-1} X^\top D_{\mu_\theta}r.
\end{align}
Let
\begin{align}
\tilde{X} \doteq \frac{X}{\norm{X}}, 
\tilde{A}_{\pi_\theta, \mu_\theta} = \tilde{X}^\top D_{\mu_\theta}(I - \gamma P_{\pi_\theta}) \tilde{X}, \tilde C_{\mu_{\theta}} = \tilde{X}^\top D_{\mu_\theta} \tilde{X},
\end{align}
then
\begin{align}
w^*(\theta) = \norm{X} (\eta I + \norm{X}^2 \tilde{A}_{\pi_\theta, \mu_\theta}^\top \tilde{C}_{\mu_\theta}^{-1} \tilde{A}_{\pi_\theta, \mu_\theta})^{-1} \tilde{A}_{\pi_\theta, \mu_\theta}^\top \tilde{C}_{\mu_\theta}^{-1} \tilde{X} D_{\mu_\theta}r.
\end{align}
We now show $w^*(\theta) / \norm{X}$ is Lipschitz continuous in $\theta$ by invoking Lemmas~\ref{lem:product-lipschitz} \&~\ref{lem:inverse-lipschitz}.
Let $D_P \in \R^{\nsa \times \nsa}$ be a diagonal matrix whose diagonal entry is the stationary distribution of the chain induced by $P$.
For any $P \in \fP$,
Assumption~\ref{assu:control-ergodic} ensures that $D_P$ is positive definite.
Consequently, $\norm{(\tilde{X}^\top D_P \tilde{X})^{-1}}$ is well defined and is continuous in $P$,
implying it obtains its maximum in the compact set $\fP$,
say $U_g$.
So $\norm{\tilde C_{\mu_\theta}^{-1}} \leq U_g$ and importantly, $U_g$ depends on $X$ through only $\frac{X}{\norm{X}}$.
Using Lemma~\ref{lem:product-lipschitz},
it is easy to see the bound and the Lipschitz constant of $\tilde{A}_{\pi_\theta, \mu_\theta}^\top \tilde{C}_{\mu_\theta}^{-1} \tilde{A}_{\pi_\theta, \mu_\theta}$ depend on $X$ through only $\frac{X}{\norm{X}}$.
It is easy to see
$\norm{(\eta I + \norm{X}^2 \tilde{A}_{\pi_\theta, \mu_\theta}^\top \tilde{C}_{\mu_\theta}^{-1} \tilde{A}_{\pi_\theta, \mu_\theta})^{-1}} \leq 1/\eta$. 
If we further assuming $\norm{X} \leq 1$,
Lemma~\ref{lem:inverse-lipschitz} then implies that $(\eta I + \norm{X}^2 \tilde{A}_{\pi_\theta, \mu_\theta}^\top \tilde{C}_{\mu_\theta}^{-1} \tilde{A}_{\pi_\theta, \mu_\theta})^{-1}$ has a Lipschitz constant that depends on $X$ through only $\frac{X}{\norm{X}}$.
It is then easy to see there exists a constant $L_w > 0$, 
which depends on $X$ only through $\frac{X}{\norm{X}}$, such that
\begin{align}
\norm{w^*(\theta_1) - w^*(\theta_2)} \leq L_w \norm{X} \norm{\theta_1 - \theta_2},
\end{align}
which completes the proof.
\end{proof}

\begin{lemma}
\label{lem diff linear q lipschitz}
The $u^*(\theta)$ defined in~\eqref{eq u star definition diff linear q} is Lipschitz continuous in $\theta$.
\end{lemma}
\begin{proof}
Recall
\begin{align}
\bar{h}(\theta) &\doteq \E_{(s, a) \sim d_{\mu_{\theta}}, s' \sim p(\cdot |s, a)}[h_\theta(s, a, s')] = \bar h_1(\theta) + \bar H_2(\theta), \\
\bar h_1(\theta) &\doteq \mqty[d_{\mu_\theta}^\top r \\ X^\top D_{\mu_\theta} r], \bar H_2(\theta) \doteq \mqty[0 & d_{\mu_\theta}^\top (P_{\pi_{\bar \theta^w}} - I) X \bar \theta^w \\ -(X^\top d_{\mu_\theta}) \bar \theta^r & X^\top D_{\mu_\theta} P_{\pi_{\bar \theta^w}} X \bar \theta^w], \\
\bar{G}(\theta) &\doteq \E_{(s, a) \sim d_{\mu_\theta}, s' \sim p(\cdot |s, a)}[G_\theta(s, a, s')] = \mqty[1 & \tb{0}^\top \\ \tb{0} & X^\top D_{\mu_\theta} X + \eta I], \\
u^*(\theta) &\doteq \bar{G}(\theta)^{-1} \bar h(\theta).
\end{align}
We now use Lemmas~\ref{lem:product-lipschitz} \&~\ref{lem:inverse-lipschitz} to compute the bounds and Lipschitz constants for several terms of interest,
which are detailed in Table~\ref{tab:diff-linear-q-bounds}.
From Table~\ref{tab:diff-linear-q-bounds} and Lemma~\ref{lem:product-lipschitz},
a Lipschitz constant of $u^*(\theta)$ is 
\begin{align}
C_u = \max\qty{1, \eta^{-1}} (\fO(\norm{X}) + \fO(L_\mu)) + \max\qty{1, \eta^{-2}} \fO(\norm{X})
\end{align}
which completes the proof.

\begin{table}[h]
\centering
\begin{tabular}{ c |c | c}
  & Bound & Lipschitz constant  \\ \hline 
$D_{\mu_\theta}$ & $1$ & $L_D$ \\ \hline
$\bar G(\theta)^{-1}$ & $\max\qty{1, \eta^{-1}}$ & $\max\qty{1, \eta^{-2}} \fO(\norm{X}^2)$ \\ \hline 
$\bar h_1(\theta)$ & $\fO(\norm{X}) + \fO(1)$ & $\fO(L_\mu)$ \\ \hline
$\bar H_2(\theta)$ & $\fO(\norm{X})$ & $\fO(\norm{X})$ \\ \hline
$\bar{h}(\theta)$ & $\fO(\norm{X}) + \fO(1)$ & $\fO(\norm{X}) + \fO(L_\mu)$ \\ \hline
\end{tabular}
\caption{\label{tab:diff-linear-q-bounds}
Bounds and Lipschitz constants of several terms, assuming $\norm{X} \leq 1, L_\mu \leq 1$.
} 
\end{table}
\end{proof}

\section{Experiment Details}
\subsection{Kolter's Example}
\label{sec kolter}
Kolter's example is a simple two-state Markov chain with $P_\pi \doteq \mqty[0.5 & 0.5 \\0.5 & 0.5]$.
The reward is set in a way such that the state-value function is $v_\pi = \mqty[1 \\ 1.05]$.
The feature matrix is $X \doteq \mqty[1 \\ 1.05 + \epsilon]$.
\citet{kolter2011fixed} shows that 
for any $\epsilon > 0, C > 0$,
there exists a $D_\mu = diag(\mqty[d_\mu(s_1) \\ d_\mu(s_2)])$  
such that 
\begin{align}
\norm{\Pi_{D_\mu} v_\pi - v_\pi} \leq \epsilon \qq{and} \norm{Xw^*_0 - v_\pi} \geq C,
\end{align}
where $w^*_0$ is the TD fixed point.
This suggests that as long as there is representation error (i.e., $\epsilon > 0$),
the performance of the TD fixed point can be arbitrarily poor.
In our experiments,
we set $\epsilon = 0.01$ and find when $d_\mu(s_1)$ approaches around 0.71,
$\norm{Xw^*_0 - v_\pi}$ becomes unbounded.
% \notey{This example can be interpreted as the difference between TD1 solution and TD0 solution in off-policy case?}
\subsection{Baird's Example}
\label{sec baird}
\begin{figure}[h]
\centering
\includegraphics[width=0.5\textwidth]{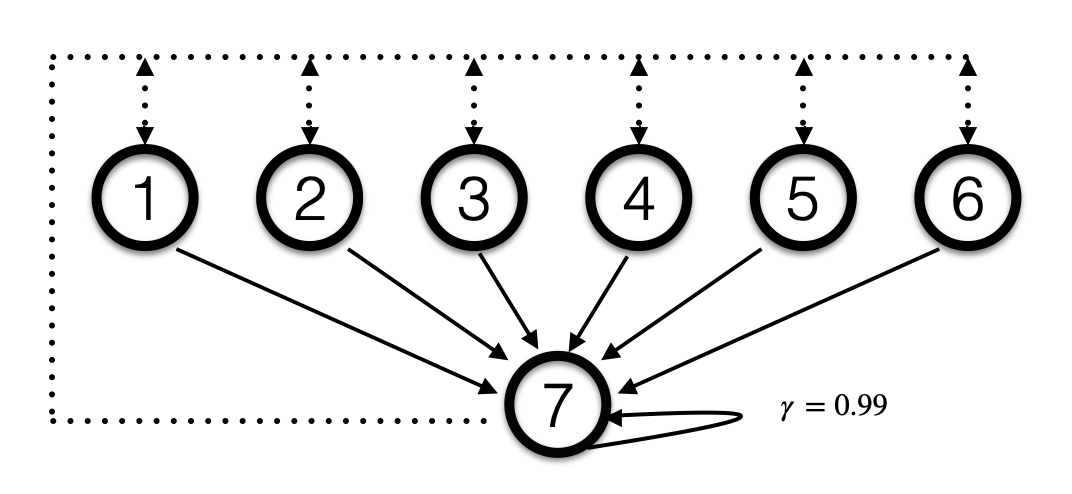}
\caption{
\label{fig baird}
The Baird's example used in Chapter 11.2 of \citet{sutton2018reinforcement}.
The figure is taken from \citet{zhang2019provably}.
At each state,
there are two actions. 
The solid action always leads to $s_7$;
the dashed action leads to one of $\qty{s_1, \dots, s_6}$ with equal probability.
The discount factor $\gamma$ is 0.99.
}
\end{figure}

Figure~\ref{fig baird} shows Baird's example.
In the policy evaluation setting (corresponding to Figure~\ref{fig:expts}b),
we use the same state features as \citet{sutton2018reinforcement}, i.e.,
\begin{align}
X \doteq \mqty [2I & \tb{0} & \tb{1} \\ \tb{0}^\top & 1 & 2] \in \R^{7 \times 8}.
\end{align} 
The weight $w$ is initialized as $[1, 1, 1, 1, 1, 1, 10, 1]^\top$ as suggested by \citet{sutton2018reinforcement}.
% \notey{Does different initialization has different effects?}
The behavior policy is $\mu_0(\text{dashed} | s_i) = 6 / 7$ and $\mu_0(\text{solid}| s_i) = 1 / 7$ for $i = 1, \dots, 7$. 
And the target policy always selects the solid action.
The standard Off-Policy Linear TD (with ridge regularization) updates $w_t$ as
\begin{align}
w_{t+1} &\gets w_t + \alpha \frac{\pi(A_t | S_t)}{\mu(A_t | S_t)}(R_{t+1} + x_{t+1}^\top w_t - x_t^\top w_t) x_t - \alpha \eta w_t,
\end{align}
where we use $\alpha = 0.01$ as used by \citet{sutton2018reinforcement}.
Note as long as $\alpha > 0$, it diverges. 
Here we overload $x$ to denote the state feature and $x_t \doteq x(S_t)$.
We use a TD version of Algorithm~\ref{alg:off-policy-td} to estimate $v_\pi$,
which updates $w_t$ as
\begin{align}
w_{t+1} &\gets w_t + \alpha \frac{\pi(A_t | S_t)}{\mu(A_t | S_t)}(R_{t+1} + x_{t+1}^\top \theta_t - x_t^\top w_t) x_t - \alpha \eta w_t, \\
\theta_{t+1} &\gets \theta_t + \beta (w_t - \theta_t),
\end{align}
where we set $\alpha = \beta = 0.01, \theta_0 = w_0$.

For the control setting (corresponding to Figures~\ref{fig:expts}c \& \ref{fig:expts}d),
we construct the state-action feature in the same way as the Errata of \citet{baird1995residual}, i.e.,
\begin{align}
X \doteq \mqty[
\mqty [2I & \tb{0} & \tb{1} & \\ \tb{0}^\top & 1 & 2] &  \tb{0}^\top \tb{0} \\
\tb{0}^\top \tb{0} & I] \in \R^{14 \times 15}.
\end{align}
The first 7 rows of $X$ are the features of the solid action;
the second 7 rows are the dashed action.
The weight $w$ is initialized as $[1, 1, 1, 1, 1, 1, 10, 1, 1, 1, 1, 1, 1, 1, 1]^\top$.
The standard linear $Q$-learning (with ridge regularization) updates $w_t$ as 
\begin{align}
w_{t+1} \gets w_t + \alpha (R_{t+1} + \max_{a'} x(S_{t+1}, a)^\top w_t - x_t^\top w_t) x_t - \alpha \eta w_t,
\end{align}
where we set $\alpha = 0.01$.
The variant of Algorithm~\ref{alg:linear-q} we use in this experiment updates $w_t$ as 
\begin{align}
w_{t+1} &\gets w_t + \alpha (R_{t+1} + \max_{a'} x(S_{t+1}, a)^\top \theta_t - x_t^\top w_t) x_t - \alpha \eta w_t, \\
\theta_{t+1} &\gets \theta_t + \beta (w_t - \theta_t),
\end{align}
where we set $\alpha = 0.01, \beta = 0.001, \theta_0 = w_0$.
The behavior policy for Figure~\ref{fig:expts}c is still $\mu_0$;
the behavior policy for Figure~\ref{fig:expts}d is $0.9 \mu_0 + 0.1 \mu_w$,
where $\mu_w$ is a softmax policy w.r.t. $x(s, \cdot)^\top w$.
For our algorithm with target network,
the softmax policy is computed using the target network as shown in Algorithm~\ref{alg:linear-q}.
\end{document}